\documentclass[11pt]{article}
\usepackage[utf8]{inputenc}
\usepackage[T1]{fontenc}
\usepackage[letterpaper,margin=1in]{geometry}
\usepackage[hyphens]{url}
\usepackage{graphicx}
\usepackage[round,authoryear]{natbib}
\usepackage{caption}
\usepackage{booktabs}
\usepackage{amsmath, amsfonts, amssymb, amsthm}
\usepackage{mathtools}
\usepackage{xcolor}
\usepackage{multirow}
\usepackage{microtype}
\usepackage[hidelinks]{hyperref}
\usepackage{cleveref}
\newtheorem{proposition}{Proposition}
\newtheorem{definition}{Definition}
\newtheorem{theorem}{Theorem}
\newtheorem*{restated}{Proposition}
\newtheorem*{restatedtheorem}{Theorem}

\title{Cross-View Correspondence Is a Measurement Intervention:\\
Two-Sided Validation for Agent Evaluation and Credit Assignment}
\author{
  Zhen Zhang \quad Ahmad Hafez \quad Amr Alanwar\\[0.5em]
  {\small School of Computation, Information and Technology}\\
  {\small Technical University of Munich, Germany}\\
  {\footnotesize\texttt{\{zhenzhang.zhang, a.hafez, alanwar\}@tum.de}}
}
\date{}
\begin{document}
\maketitle

\begin{abstract}
Agent evaluations and trace-based learning often compare outputs across transformed views through a
post-response correspondence treated as neutral preprocessing. We show that this correspondence is a
measurement intervention: omitting it can manufacture sensitivity, an over-aggressive map can
manufacture invariance, and multiple optimal correspondences can leave mechanism labels and signed
learning credit unidentified. We develop a validity theory and audit with three components: two-sided
validation of nuisance removal and response preservation, all-optima identification of downstream
conclusions, and uncertainty propagation after validity is established. We characterize the linear
feasibility boundary for response-preserving nuisance removal, compute sharp ranges over exact-optimum
correspondence sets, and give a distribution-free certificate that retains a credit coordinate only
when all exact optima agree on its nonzero sign. Across public code and SQL pipelines, two
deterministic optimal tracebacks disagree on temporal localization for $55.9\%$ of $1{,}586$ nonzero
trajectory pairs; two frozen $800$-rollout tool-use audits, including a task-and-seed-disjoint
replication, expose exact-optimum reversals of intended turn-level credit, although a clean public
quick-start subset shows none. A pre-registered transport gate failed on natural responses; frozen
corrected and held-out controls then show that a map calibrated only on benign examples erases every
retained harmful response, while two-sided validation selects response-preserving alternatives.
Cross-view correspondence must therefore
be declared, validated, and propagated into uncertainty before agent evaluation or credit assignment
supports a point conclusion.
\end{abstract}

\section{Introduction}
\label{sec:intro}
An optimal matching can support opposite scientific conclusions. Suppose one reference tool call
ties with two predicted calls. Either match preserves total score, yet credits a different turn and
can induce opposite local credit after reward compilation. The objective is optimal but does not
identify the downstream conclusion.

This problem arises whenever a pipeline transforms a prompt, tool schema, interface, repository, or
observation before comparing agent behavior. The evaluator must then decide which elements correspond
across views. We use \emph{correspondence} for the full interface: cross-view transport, within-output
matching, and completion of exact ties. Our central claim is that this interface is a
\emph{measurement intervention}, not neutral preprocessing. Unless a specific convention is part of
the estimand, the scientific object is the set of conclusions supported by all legal
correspondences, not the single conclusion returned by one solver run.

The intervention can fail in three directions (\Cref{fig:overview}). Omitting an appropriate
transport leaves representational nuisance behind and manufactures sensitivity. A map calibrated
only on benign examples can erase response-bearing differences and manufacture invariance. Finally,
several exact-optimal matchings can preserve the same scalar score while disagreeing on a mechanism
label or signed learning credit. A cross-view point claim is therefore solver-independent only when
the correspondence rule is declared, shown to remove nuisance without suppressing response, and
supported by every legal exact optimum.

\begin{figure*}[t]
\centering
\includegraphics[width=\textwidth]{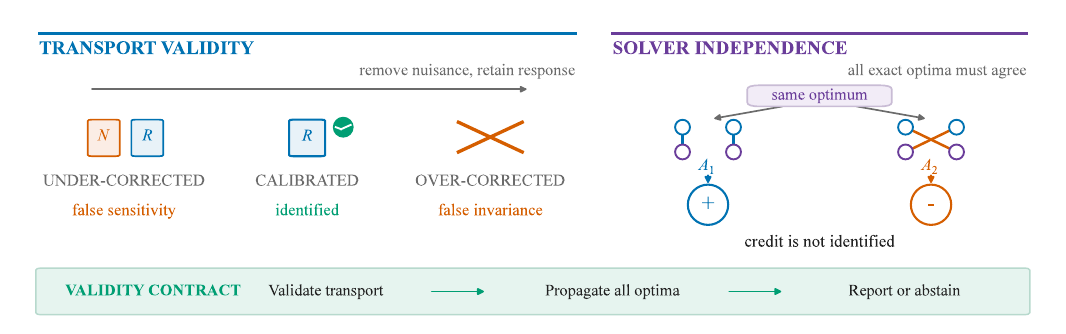}
\caption{Correspondence validity has two parts. Transport must remove nuisance without erasing
response, and all exact-optimal correspondences must support the same downstream conclusion. Our
contract validates both requirements before returning a point, a common direction, or abstention.}
\label{fig:overview}
\end{figure*}

Prior work separately studies declared transformations
\citep{formatspread2024,rotbench2024,guirobustbench2026,repomirage}, unknown maps and acceptable
matches \citep{gulrajani2022identifiability,lember2014optimal,morucci2022robust}, and common-direction
or structured inference \citep{sener2018multiobjective,jaggi2013frankwolfe,koller2009pgm}. What is
missing is one contract that validates the map, propagates every exact optimum, and certifies the
resulting learning decision.

Our \emph{correspondence contract} closes this gap; its contributions are as follows:
\begin{itemize}
    \item A two-sided validity contract combines transport validation, all-optima identification,
    and subsequent uncertainty propagation.
    \item Legal-update geometry and a compiler-dependent frontier provide safe credit, common
    progress, compact certification, and a source-realizable hardness boundary.
    \item Natural audits expose hidden diagnosis and credit policies, repair a public strict metric,
    and preserve failed, corrected, and held-out falsification records.
\end{itemize}
The scope is not that every pipeline fails, but that a point conclusion requires a declared and
validated correspondence contract.

\section{Correspondence Contracts and Two-Sided Validity}
\label{sec:validity}
\label{sec:model}
Two levels of correspondence appear in the pipeline. We call the cross-view map $\Phi_T$ a
\emph{transport}, and an assignment $A$ between elements of two outputs a \emph{matching}.
The transport must be validated first; any remaining exact matching choices must then be
propagated to the downstream conclusion. The two-way tie of \Cref{sec:intro} runs throughout.

\paragraph{The cross-view target.}
An audit fixes a task $g$ and a base object $R_0$ (prompt, tool schema, screen, repository), forms
$R_0'=T(R_0)$ with a perturbation $T$ intended to be semantics-preserving, and runs the same random
agent $A_\omega$ on both. Let $U=A_\omega(R_0)\in\mathcal Z_0$ and
$V=A_\omega(T(R_0))\in\mathcal Z_T$ be paired outputs. A transport
$\Phi_T:\mathcal Z_T\to\mathcal Z_0$ aligns the perturbed output with the original view. For a
declared loss or distance $d$ on $\mathcal Z_0$, the response of interest is
\begin{equation}
J_T(\Phi_T)=\mathbb E\!\left[d\!\left(\Phi_T(V),U\right)\right].
\label{eq:transportestimand}
\end{equation}
Writing the raw gap $\mathbb E[d(V,U)]$ already assumes that the identity inclusion is a justified
transport. For ASTs, renamed tools, GUIs, and trajectories, it often is not. For observed law $P$ and
declared class $\mathcal P_T$, the sharp population response set is
\begin{equation}
\mathcal I_T(P,\mathcal P_T)=
\{\mathbb E_P[d(\Phi(V),U)]:\Phi\in\mathcal P_T\}.
\label{eq:transportset}
\end{equation}
This set is sharp because $P$ does not select $\Phi$: every declared $\Phi$ defines a compatible
structural world and every admitted world uses one such map. A point claim therefore requires a
declared map or agreement over the set; more samples cannot resolve an undeclared convention.

\paragraph{Consuming exact ties.}
Even after a transport has been declared, a lower-level matching objective can admit several exact
assignments. The next definition records both the objective and what the pipeline does with that set.

\begin{definition}[Correspondence contract]
For input $x$, let $Q_x(A)$ be a correspondence objective with exact optimizer set
$\mathcal A_Q^\star(x)$, and let $h(A)$ be the downstream readout. A correspondence contract
$\mathfrak C=(Q,\rho,h)$ declares both the objective and how its optimizer set is consumed. Policy
$\rho$ may be a deterministic selector $s$, a stochastic law $\pi$, or the requirement that the
scientific conclusion be invariant over all admissible exact optima.
\end{definition}
A deterministic tie-break defines a reproducible selector-relative target, but not a property
independent of solver convention. A stochastic law defines another target and must state whether
matching is sampled before or after downstream normalization. Soft matching changes $Q$ and hence
the contract; it does not identify the original exact-optimum claim.

\paragraph{Why calibration must be two-sided.}
Let $\mathcal C_0$ contain \emph{null controls} $(u,v_0)$ known to differ only by the inserted view,
and define null leakage
$e_0(\Phi)=\sup_{(u,v_0)\in\mathcal C_0}d(\Phi(v_0),u)$. Small $e_0$ prevents false
sensitivity. It does not prevent false invariance. For a declared real-valued behavioral witness
$q_T:\mathcal Z_T\to\mathbb R$ and its original-view counterpart $q_0$, call $\Phi$
\emph{$q$-faithful} on $\mathcal V$ when $q_0(\Phi(v))=q_T(v)$ for every $v\in\mathcal V$.
When response magnitude is measured by a metric $d_T$, define the restricted response gain
\begin{equation}
\kappa_{\mathcal G}(\Phi)=
\inf_{(v_0,v_1)\in\mathcal G}
\frac{d(\Phi(v_1),\Phi(v_0))}{d_T(v_1,v_0)},
\label{eq:responsegain}
\end{equation}
over the positive-control response-pair set $\mathcal G$ with nonzero denominator (distinct from the
response subspace $\mathcal R$ of part (c) and the transport class $\mathcal P_T$).

\begin{definition}[Two-sided alignment certificate]
For declared null controls $\mathcal C_0$, response-bearing controls $\mathcal G$, and thresholds
$\delta_0\ge0$ and $\kappa_0>0$, a transport is \emph{certified} when
$e_0(\Phi)\le\delta_0$, $\kappa_{\mathcal G}(\Phi)\ge\kappa_0$, and every required behavioral
witness is preserved. The certificate is relative to these controls and witnesses; it is not a claim
of global semantic equivalence.
\end{definition}

The following result explains why good reconstruction on benign examples is insufficient and gives a
sharp linear boundary for simultaneously removing nuisance and retaining response.

\begin{theorem}[Two-sided validity and its linear boundary]
\label{thm:alignmenttrap}
\emph{(a) One-sided insufficiency.} Over an unrestricted transport class, any values of $e_0$
on $\mathcal C_0$ are compatible with $\kappa_{\mathcal G}=0$ on an untested response pair and with
failure of $q$-faithfulness. Thus benign-only calibration cannot certify an invariance conclusion.
\emph{(b) Witness certificate.} If $\Phi$ is $q$-faithful and $q_0$ is $L$-Lipschitz under $d$, then
$|q_T(v_1)-q_T(v_0)|\le Ld(\Phi(v_1),\Phi(v_0))$; transported invariance at radius $\tau$ therefore
implies witness change at most $L\tau$.
\emph{(c) Linear feasibility.} In a finite-dimensional Hilbert space, let $\mathcal N$ and
$\mathcal R$ be nuisance and response subspaces. A linear transport that kills $\mathcal N$ and has
positive restricted gain on $\mathcal R$ exists iff $\mathcal N\cap\mathcal R=\{0\}$. For orthogonal
projection onto $\mathcal N^\perp$, the gain is
$\sin\theta_{\min}(\mathcal R,\mathcal N)$, the sine of the smallest principal angle
\citep{bjorck1973angles}.
\end{theorem}
The proof mechanism is short. For (a), leave every benign control fixed but map one untested response
to its baseline; benign error is unchanged while response gain becomes zero. Part (b) follows directly
from faithfulness and Lipschitzness. For (c), a shared nonzero response--nuisance direction must be
annihilated by every nuisance-killing map; conversely, projection onto $\mathcal N^\perp$ is injective
on $\mathcal R$, and compactness of its unit sphere gives minimum gain
$\sin\theta_{\min}(\mathcal R,\mathcal N)$. Thus the theorem combines elementary validity
implications and standard principal-angle geometry into a two-sided gate at the correspondence
interface; App.~A.1--A.2 give the identification and validity proofs in full. These arguments prove all three parts. Part (c) bounds linear transports only: the nonlinear source
transformations audited in \Cref{sec:reanalysis} are admitted by frozen controls, not by (c).
The resulting refusal is explicit: invalid transport, correspondence-dependent conclusion, or
statistical inconclusiveness.

\section{Solver-Independent Credit: Geometry and Complexity}
\label{sec:updatebody}
\subsection{Legal Update Bodies and Safe Decisions}
Validating the cross-view transport does not choose among tied output correspondences. We therefore
carry every exact optimum through the same reward compiler and policy-score map, rather than treating
one solver completion as the target.

Let $\mathcal F=\mathcal A_Q^\star(x)$ be the legal exact-correspondence set for a training group
(the exact-optimum fiber), finite in every audited pipeline, so every body below is a compact
polytope. Let $r(A)$ be its fine-grained reward vector, $N$ the declared reward compiler, and
$S_\theta$ the policy-score linear map at parameters $\theta$. Each optimum produces
$g_\theta(A)=S_\theta N(r(A))$. Their convex hull is the \emph{legal update body}:
\begin{equation}
\mathcal K_\theta=\operatorname{conv}\{g_\theta(A):A\in\mathcal F\}.
\label{eq:updatebody}
\end{equation}
A set and its convex hull share support functions, so convexification changes neither linear support
nor the existence of a direction with positive inner
product against every legal update. In the running two-way tie, $\mathcal F$ has two elements and
$\mathcal K_\theta$ is the line segment joining their compiled updates. The decision is whether that
segment collapses to one point, stays strictly on one side of the origin, or reaches the origin.

\begin{proposition}[Solver-independent update trichotomy]
\label{prop:updatetrichotomy}
The update is solver-identified iff $\mathcal K_\theta$ is a singleton. A unit direction $u$ with
$\inf_{g\in\mathcal K_\theta}\langle u,g\rangle>0$ exists iff
$0\notin\mathcal K_\theta$. If $x^\star$ is the closest point in
$\mathcal K_\theta$ to the origin, then
$u^\star=x^\star/\|x^\star\|$ certifies margin at least $\|x^\star\|$.
If $0\in\mathcal K_\theta$, no strict common-progress direction exists.
\end{proposition}
This is the classical minimum-norm separation result used in multi-objective optimization
\citep{fliege2000steepest,desideri2012mgda,sener2018multiobjective}; the new object is the body induced
by compiling every legal correspondence. If each $F_A$ is $L$-smooth with certified margin $m$, then
$F_A(\theta+\eta u)\ge F_A(\theta)+\eta m-L\eta^2/2$; hence every legal objective improves locally
for $0<\eta<2m/L$ (App.~A.6).

The middle case can be certified without enumerating the entire body. A linear-support oracle for
$\mathcal K_\theta$ is enough to construct $x^\star$ with the classical fully corrective
Frank--Wolfe/Gilbert procedure \citep{gilbert1966minnorm,jaggi2013frankwolfe}. At iterate $x$, one
support query gives dual gap $\zeta$ and the explicit certificate
\[
\min_{g\in\mathcal K_\theta}\left\langle x/\|x\|,g\right\rangle
\ge \|x\|-\zeta/\|x\|.
\]
After median-norm rescaling, all $15$ natural bodies reached $\zeta\le10^{-8}$ in at most nine iterations (cap $1{,}000$); direct full-hull QPs agreed within $8.33\times10^{-12}$.
Coordinatewise diagnosis asks a different question. It keeps coordinate $j$ only when every
$A\in\mathcal F$ has the same nonzero sign. This rule is pointwise maximal because any additional
coordinate is zero or has the opposite sign under a legal optimum.
Its vector need not lie in $\mathcal K_\theta$ and is not automatically a legal full-gradient update.
Neither procedure invents a probability law over ties.

The compiler must be inside the body definition.
\begin{proposition}[Relative-credit normalization trilemma]
\label{prop:normalizationtrilemma}
Let $N$ be $C^1$ on a connected translation-stable domain of nonconstant $n$-sample reward groups,
$n\ge3$, and let $\mathbf1$ be the all-ones vector. Translation neutrality
$N(x+c\mathbf1)=N(x)$ together with current-peer isolation
$\partial N_i/\partial x_j=0$ for $i\ne j$ forces $N$ to be constant. On any ray-stable centered
subdomain, exact scale neutrality $N(ay)=N(y)$ implies
$DN(ay)=a^{-1}DN(y)$, so a nonconstant $N$ has no uniformly Lipschitz extension to zero dispersion.
Finally, sampling an exact optimum before normalization and normalizing its fiber mean agree for every
finitely supported law exactly when $N$ is affine on the fiber hull.
\end{proposition}
Differentiating translation neutrality yields $DN(x)\mathbf1=0$. Isolation makes the Jacobian
diagonal, so it must vanish. Scale neutrality similarly yields $aDN(ay)=DN(y)$; equality for every
two-point law is segment preservation. The requirements therefore trade off peer dependence, stability
near ties, and the order in which correspondence uncertainty is consumed (App.~A.3--A.5).

\subsection{A Compiler-Dependent Tractability Frontier}
\label{sec:frontier}
An update body may contain exponentially many completions without being hard to audit. The decisive
object is the compiler: it determines whether local tie choices remain separate or become globally
coupled.

\paragraph{Affine compilers.}
Suppose fiber $i$ has at most $q$ exact states. If $r$ and $N$ are affine on the product of exact
assignment faces, \Cref{eq:updatebody} has a polynomial-size extended formulation: support is a
linear program and the closest point is a convex QP. Assignment polytopes are integral
\citep{birkhoff1946tres}, so each face has exactly the exact matchings as vertices and the linear
program is exact rather than a relaxation. Independent exact fibers give a Minkowski sum
of local polytopes; constrained-zonotope representations support exact Minkowski sums through simple
identities \citep{scott2016cz}. Thus exponentially many completions need not imply
exponential auditing.

\paragraph{Local nonlinear compilers.}
For a turn-local nonlinear compiler, a projected query factors as
$c+\sum_t\psi_t(A_{S_t})$, where $S_t$ contains the fibers whose choices change turn $t$.
The \emph{compiler interaction graph} joins fibers that co-occur in a factor. Standard min-sum
elimination computes exact lower and upper support in
$\operatorname{poly}(n)q^{w+1}$ time when this graph has treewidth $w$
\citep{koller2009pgm}. The graph is a property of the composed reward pipeline, not of the matching problems viewed
separately. In all eight eligible natural groups its exact treewidth is at most three; min-sum
elimination matches exhaustive lower support in all $45$ frozen projected queries to
$1.19\times10^{-19}$.

\paragraph{Shared normalization.}
Operationally, flag a compiler for joint all-optima auditing when an exact fiber is non-singleton,
a group statistic is recomputed after tie completion, and normalization is non-affine on the fiber
hull. The flag does not itself imply hardness. For $z\in\mathbb R^n$, define
$\operatorname{Std}_\epsilon(z)=(z-\bar z\mathbf1)/(s(z)+\epsilon)$, where
$s(z)^2=(n-1)^{-1}\sum_i(z_i-\bar z)^2$. The audited dual-level compiler is
$N_{t,i}=\tfrac12\operatorname{Std}_\epsilon(G_t)_i+\tfrac12\operatorname{Std}_\epsilon(R)_i$,
where $G_t$ collects discounted turn-$t$ returns and $R$ scalar trajectory rewards. The theorem
concerns this exact compiler: hardness comes from composition, not from finding one matching.

\begin{theorem}[Group standardization is an Ising compiler]
\label{thm:isingcompiler}
Fix rational discount $\gamma=9/10$ and a positive rational stabilizer $\epsilon$ (audited at the
released $10^{-6}$). For every
nonempty unweighted graph $G=(V,E)$ with $m=|V|$, there is a polynomial-size family of literal tool-call
similarity instances such that: each ambiguous assignment fiber has exactly two exact optima; both
preserve the same scalar trajectory reward; every trajectory total is equal; only two deterministic
anchor trajectories carry the queried policy-score derivative; and the compiled gradient coordinate
is
\[
\begin{aligned}
g_G(y)&=C_G+J_G\sum_{(i,j)\in E}y_i y_j,\\
y_i&\in\{-1,+1\},\qquad J_G>0.
\end{aligned}
\]
Adjacent cut sizes stay $\Omega(m^{-14})$ apart in support, so on a finite-bit promise formulation
inverse-polynomial-additive lower support is NP-hard and universal all-optima threshold
certification is coNP-hard.
\end{theorem}

\paragraph{Proof idea.}
A cyclic assignment gadget gives each vertex two exact optima $y_i\in\{-1,+1\}$. Write
$M=|E|$, $T=M+1$, and $N=m+2$. At the first $M$ turns prescribe discounted returns
$d^+_{i,e}=\mathbf1\{i\in e\}$ and $d^-_i=-d^+_i$. Set
$r_t=d_t-\gamma d_{t+1}$, $r_{T-1}=d_{T-1}$, and choose
$d_{T-1}=-[d_0+(1-\gamma)\sum_{t=1}^{T-2}d_t]/(1-\gamma)$.
Then each sign realizes the prescribed returns but has zero raw-reward total; two $\pm1$ anchors use
the same compensation. At $\theta=0$, one shared scalar two-action softmax realizes score derivatives
$+1,-1,0$ from realized/alternative feature pairs $(+1,-1),(-1,+1),(0,0)$. Reusing $y_i$ at every
incident edge gives variable return group
$\delta(y_i,y_j,0,\ldots,0,+1,-1)$. Its unscaled variance is
$v_e(y)=[4-(y_i+y_j)^2/N]/(N-1)$, hence
$v_{\rm cut}=4/(N-1)>v_{\rm uncut}=4/N$. Centering removes the common offset, while the two
anchors convert standardization into
$\alpha(v)=\delta/(\delta\sqrt v+\epsilon)$. Equal trajectory totals null the global branch, and
$\alpha_{\rm cut}<\alpha_{\rm uncut}$ gives
$J_G=(\alpha_{\rm uncut}-\alpha_{\rm cut})/2>0$. Summing edges yields the displayed
antiferromagnetic Hamiltonian, whose lower support encodes maximum cut \citep{garey1976simplified}.
For finite-bit source realization, write each rational raw vector as integers $z_t=10r_t$, set
$C_0=(T-1)\max_t|z_t|+1$, and place $C_0+z_t$ on identity cycle edges and $C_0-z_t$ on shifted-cycle
edges. Pair-isolating parameter keys realize similarities $(2+c_{ij})/(L+2)$ with $L=2TC_0$; every
partial assignment improves to a full one, and any other permutation uses at most $T-1$ cycle edges,
so only the identity and cyclic shift attain $TC_0$.
Their reward vectors are $K\mathbf1\pm\delta r$, where $K=(2+C_0)/(L+2)$ and
$\delta=10/(L+2)$, so both totals equal $TK$. Here $T=O(m^2)$, $L=O(m^6)$, and the return scale is
$\Omega(m^{-6})$; the cut/uncut variance gap then
gives adjacent support separation $\Omega(m^{-14})$, so a polynomial-bit dyadic midpoint completes the
Karp reduction. Independently, all $813$ no-isolate graphs through five vertices ($25{,}264$ sign
assignments), $105$ literal call dictionaries, and extracted upstream code verify the three reduction
layers (maximum numerical error $1.61\times10^{-6}$). The result concerns data-dependent shared
standardization, not a fixed external scale or centering alone. App.~A.7 gives the full construction
and gap bound; App.~B.8 gives the conformance ledgers.

\begin{table*}[t]
\centering
{\small
\setlength{\tabcolsep}{4pt}
\begin{tabular}{@{}p{.20\textwidth}p{.24\textwidth}p{.23\textwidth}p{.25\textwidth}@{}}
\toprule
Contract / compiler & Solver-invariant object & Exact method & Boundary \\
\midrule
Declared selector & one assignment and update & assignment + compiler & polynomial, selector-relative \\
Affine on exact faces & compact polytope & LP support; QP closest point & polynomial, all optima \\
Local nonlinear, width $w$ & finite factor graph & min-sum elimination & $\operatorname{poly}(n)q^{w+1}$ \\
Shared group standardizer & coupled exact-tie states & support / universal threshold & NP-hard / coNP-hard \\
\bottomrule
\end{tabular}
}
\caption{Complexity of all-optima auditing under different correspondence contracts and reward
compilers. Affine and low-treewidth regimes admit exact classical algorithms; shared
standardization reaches the hardness boundary in \Cref{thm:isingcompiler}.}
\label{tab:frontier}
\end{table*}

The gap is operational: one exact matching and its update remain polynomial in the construction,
but certifying a threshold over every equally optimal completion is coNP-hard. Ising models, MAX-CUT,
treewidth algorithms, and non-singleton bilevel optimization are not new; the contribution is the
restricted source-level composition that creates this gap. It yields a constructive dispatch rule:
use compact affine certification, exact low-treewidth elimination, or bounds/abstention instead of
treating one cheap run as a certificate. Bounding is also the cheaper semantics: near-optimal
relaxation makes additive auditing NP-hard \citep{handler1980dual}, and exact averaging over an
optimal matching set is \#P-hard \citep{valiant1979permanent,jerrum2004permanent} (App.~D).

\section{Natural Consequences and Exact Repairs}
\label{sec:naturalaudits}
Three audits test whether a readout is constant over all optimal correspondences---a temporal
diagnosis, signed turn-level credit, and an advertised strict trajectory score---and a fourth tests
transport validity. Pre-registered reanalysis, disjoint replication, a failed prospective gate,
post-diagnostic correction, and pre-outcome held-out transfer are never pooled. Together they
establish natural ambiguity and exact repairs, not final-policy effects or natural-case hardness.

\subsection{Hidden Correspondence Dissolves Temporal Claims}
\label{sec:naturalalignment}
\citet{consistencytestable2026} compare base and perturbed agent trajectories with action-token
Levenshtein metrics and weight later edits more heavily because early uncertainty can be expected
while later deviations indicate unreliable execution. The resulting mechanism label depends on the
selected correspondence. We leave their public action
tokens, task outcomes, and unit edit objective unchanged. For each paired trajectory, let
$\mathcal A^\star$ contain every minimum-cost alignment. A downstream diagnosis $h(A)$ is identified
by the metric only when it is constant on this set, $\min_{A\in\mathcal A^\star}h(A)=\max_{A\in\mathcal
A^\star}h(A)$; otherwise its reported value is fixed by a hidden tie-break rather than by the data.
Every optimal path carries the same edit count, so additive $\min$/$\max$ semirings over edge
positions give sharp bounds on the mean position, and the same two-pass program over the
optimal-path DAG decides constancy. We preregister
early ($<.40$), late ($>.50$), and \emph{alignment-unidentified} (the sharp range reaches both
sides), requiring agreement under source-index and symmetric DP-state normalizations.

Across $1{,}699$ released SWE-bench and Spider2-DBT pairs, $1{,}586$ have nonzero edit distance.
The threshold-free result is already substantial: deterministic left- versus right-priority optimal
tracebacks give different temporal labels for $55.9\%$ of non-null pairs, and mean sharp-range width
is $.466$ for Codex versus $.094$ for the heterogeneous-action OpenHands control. Under the
pre-registered early/late rule, $47.5\%$ are alignment-unidentified under both normalizations. The
ambiguity changes an aggregate conclusion, not only selected examples (\Cref{fig:optimalalign}). For
Spider Codex \texttt{header\_shuffle}, the sharp cell-mean range is $[.221,.799]$ rather than a point
near $.217$; for SWE Codex \texttt{linear\_mcp}, it is $[.286,.753]$. The same minimum edit count
therefore supports either localization. The Spider OpenHands control is much tighter,
$[.498,.675]$, and never reaches the early regime.

\begin{figure*}[t]
\centering
\includegraphics[width=.88\linewidth]{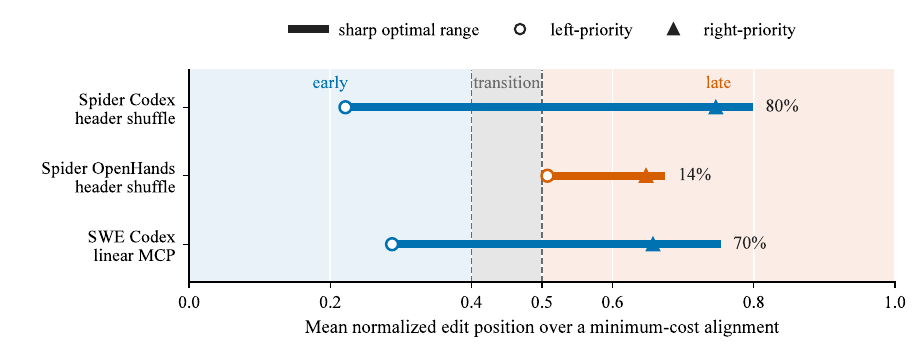}
\caption{Sharp temporal-localization ranges over minimum-cost trajectory alignments. Horizontal
segments show the cell-mean edit position over all per-pair exact optima; symbols denote two
deterministic optimal tracebacks. Percentages report pairs whose ranges span both the early
($<.40$) and late ($>.50$) regimes. Both homogeneous-action Codex cells cross the two regimes,
whereas the heterogeneous-action OpenHands control does not reach the early regime.}
\label{fig:optimalalign}
\end{figure*}

The cleanest witness needs no model at all. Comparing one released M3-Bench trajectory \emph{to
itself} under its global Hungarian correspondence \citep{m3bench2025} yields $32$ exact optima---
repeated identical names and canonicalized arguments hold every permutation of those copies at cosine
one, so the ties are structural, not numerical---whose Order Consistency ranges over the full $[0,1]$
and whose Step Coherence ranges over $[.43,1]$. No model, judge, or sampling noise explains that
spread. A non-identical cross-revision pair spans $[.44,.89]$, and the defect survives both dataset revisions.
Across $260$ task/assignment combinations, our implementation matches the official evaluator
(maximum error $3.79\times10^{-8}$). The auditors are exact for additive statistics on optimal-path
DAGs and explicit optimal-assignment faces; general nonlinear statistics may require enumeration or
certified bounds. This establishes metric non-identification across backends and domains, but not a
model-ranking reversal, because M3-Bench releases no per-model predictions (App.~B.2).

All frozen temporal controls pass: dominant-action share predicts range width (Spearman $\rho=.606$,
$95\%$ CI $[.561,.647]$). A preregistered endpoint convention gave $44.6\%$ ambiguity; after matching
the source paper's state convention, the corrected value is $47.5\%$. Both versions reverse the same
two aggregate cells, and frozen threshold sweeps span $43.2$--$49.4\%$. The audit changes no outcome,
action count, or distance. Full ledgers are in App.~B.1.

\subsection{Hidden Correspondence Reverses Assigned Credit}
\label{sec:matchtir}
MatchTIR \citep{qu2026matchtir} assigns dense turn rewards by bipartite matching between predicted and
ground-truth tool traces, then combines turn- and trajectory-level advantages. Matching objective
optimality does not imply credit identification: tied optima can distribute the same total similarity
over different predicted calls. Before generation we froze the checkpoint, source revision, task rows,
four rollout seeds, temperature $1$, top-$p$ $1$, six turns, $768$ new tokens per turn, the exact
enumerator, and the promotion criteria. The frame holds $100$ rows with repeated
ground-truth tool names and $100$ without; the released $4$B checkpoint produced $800/800$
trajectories. For trajectories with two to eight parsed calls, a bitmask program enumerates every
unique per-call reward vector attaining the exact maximum-weight partial assignment. Frozen
materiality requires two distinct predicted-call JSONs and reward width at least $.10$. In each
four-rollout group, rewards are aggregated into turns, discounted at $.9$, standardized across
rollouts at both turn and trajectory levels, and equally averaged. This is a consequence audit on the
training distribution, not a held-out capability evaluation.

Our primary target is the multi-turn mechanism described in the source. The public quick-start does
not visibly enable its default-off mask: reconstructing that path yields $6/20$ sign-flip groups over
all material drivers, $1/9$ for all-distinct drivers, and $0/7$ after also removing tool errors. We
therefore test whether the source-described compiler identifies local credit on natural rollouts,
without inferring that an undocumented training run enabled the branch or that final policy quality
changed; that last level remains unmeasured.

Under the source-described multi-turn branch, $495$ trajectories are multi-call and $93$ have material
exact-optimum reward width ($18.8\%$, row-clustered $95\%$ CI $[14.1,24.0]\%$); every completion is
exactly optimal. The intended multi-turn advantage reverses sign in $14/20$ affected task groups. As diagnostic coverage, the
all-optima certificate retains $332/437=76.0\%$ of canonical nonzero coordinates; its vector is not
claimed to be a realizable full-gradient update. A task-and-seed-disjoint $800$-rollout audit finds
$14/22$ material groups with a reversal and retains $283/402=70.4\%$ of exact strict signs with zero
false signs. Seven groups contain a peer-only cross-sample sign-reversal witness. The overall frozen
verdict is nevertheless \textsc{Weak}: only $5/11=.455$ witnesses clear both $.05$ margins (required
$.60$), and the frozen risk score has AUC $.589$ (required $.70$).

The peer mechanism has an exact certificate. If independently matched rewards satisfy
$x_i\in[\ell_i,u_i]$, $w_i=u_i-\ell_i$, and $C_i=x_i-n^{-1}\sum_kx_k$, then
$C_i\in[L_i,U_i]$, where
\[
\begin{aligned}
L_i&=((n-1)\ell_i-\textstyle\sum_{j\ne i}u_j)/n,\\
U_i&=((n-1)u_i-\textstyle\sum_{j\ne i}\ell_j)/n,\\
U_i-L_i&=((n-1)w_i+\textstyle\sum_{j\ne i}w_j)/n.
\end{aligned}
\]
The endpoints are attained by independent endpoint choices. Thus even $w_i=0$ inherits peer
uncertainty, while division by a positive empirical scale preserves the centered-credit sign.

\paragraph{The hidden completion is a temporal credit policy.}
The released routine enumerates score edges in predicted-call order, stably sorts them by score, and
greedily accepts them. Predicted calls are collected chronologically, so equal-score competition
inherits an undeclared early-call priority not present in the matching objective. In a separately
pre-registered analysis of the two frozen ledgers, we reversed only the call array presented to this
selector and mapped its rewards back to the unchanged calls. Every reversed selection still attains
the exact objective; every call, turn, and total reward is preserved.

The direction holds in all three public source strata and under aggregation by actual occupied turn,
whereas $128$ fixed random input permutations per trajectory reduce mean positional bias to about
$-.001$. An independent implementation, a bitmask exact solver, and direct execution of the four
relevant upstream functions agree on every retained trajectory, with original and reversed totals
differing by at most $8.9\times10^{-16}$. Thus an order-free objective compiles into an
order-sensitive learning policy solely through its hidden completion.

\paragraph{Pricing the declarable alternatives.}
The chronological selector is reproducible but favors earlier calls by $.196$ $[.149,.246]$ and
$.174$ $[.133,.219]$, with negative slopes on $98.9\%$ and $100\%$ of trajectories. A uniform law over
distinct exact-optimal reward vectors flips $20.1\%$ of canonical nonzero coordinates in expectation;
normalizing before versus after averaging disagrees in sign in $7/20$ and $8/22$ groups. The all-optima certificate instead keeps $76.0\%$ with zero false signs. These are local-credit
non-identification results, not reported-score errors or measured policy-level effects
(App.~B.3--B.6).

\subsection{A Natural Strictness Failure and Exact Repair}
\label{sec:earthstrict}
Earth-Agent \citep{feng2026earthagent} scores reasoning fidelity by ``Tool-Exact-Match''. The source
paper defines it as the longest common prefix over the ground-truth length, which gives a
perfect value to any extension of a complete expected prefix, yet the paper reads the score as
penalizing irrelevant extra steps.

For a nonempty reference $r$ of length $m$ and response $p$ of length $n$, let $\ell(r,p)$ be their
exact name-prefix length. The source score is $s_{\rm ref}=\ell/m$, so
$s_{\rm ref}(r,r\mathbin{\|}z)=1$ for every extra suffix $z$. Our frozen two-sided repair,
$s_2=\ell/\max\{m,n\}=\min\{\ell/m,\ell/n\}$, equals one if and only if the complete name sequences
agree, since $\ell\le\min\{m,n\}$ forces $m=n=\ell$; full-call equality gives the analogous parameter
certificate. This is conservative exactness
accounting, not a claim to a new sequence metric.

Before reading prediction contents, we froze the public commit, the evaluator-defined RGB slice,
26 configurations, the repair, clustered inference, and six decision gates. Across 1,532 natural
task-configuration cells, 218 non-exact traces receive source score one
($14.23\%$, task-clustered $95\%$ CI $[7.58,21.80]\%$). The defect appears in $22/26$
configurations and $12/13$ model families; $26.17\%$ of cells change by at least $.05$. The repair
creates nine pairwise ranking reversals and changes the top configuration from DeepSeek-V3.1 IF to
GPT-4o AP. A source-parity script reproduces every saved score exactly; this CPU-only audit uses no model
generation, LLM judge, or human labels. The top change survives all $59$ leave-one-task-out omissions (App.~B.11).

A post-hoc complete-case check retains defects in $23/26$ configurations and all $13$ families, with
18 pairwise reversals. In held-out MCPEval \citep{liu2025mcpeval}, 15/248 cells are false-perfect and
one pair reverses, but frozen prevalence gates fail and its binary strict-success field already rejects
extra calls. These metric-specific results neither judge extra-call semantics nor end-to-end accuracy.

\paragraph{Natural bodies are nontrivial but decision-redundant.}
Across $571$ legal updates from $15$ low-width tasks, median maximum pairwise angle is
$25.13^\circ$, $10/15$ tasks exceed $15^\circ$, and one reaches $161.90^\circ$. Exact support still
certifies the canonical update's same $14$ common-progress decisions and abstains once; the robust
direction raises worst-case unit margin by median $6.4\%$ and at most $28.3\%$. This post-hoc slice
validates cheap low-width dispatch and certificate quality, not policy improvement (App.~B.7).

\subsection{Pre-Registered Two-Sided Transport Validation}
\label{sec:reanalysis}
To test the map itself, a pre-registered audit froze transport, null threshold $\tau=.05$, inference,
one agent (gpt-5.5 at high reasoning effort, Codex CLI $0.144.5$), and $24$ paired three-stage
trajectories; all $72$ stages completed. Agent-free calibration gives zero
residue on $133$ held-out functions although $T$ moves the region by median $.134$. On pylint, raw
horizon-1 distance $.175$ matches the agent-free residue $.176$ and is $60.7\times$ the largest
within-view seed distance, so a raw point rule reads sensitivity; the transported distance $.019$ is
equivalent at $\tau$, and both substrates appear invariant. The natural-response gate therefore failed. Later controlled and frozen
checks diagnose this failure rather than retroactively pass it: benign reconstruction alone cannot
establish response preservation.

The later gate compares destructive \emph{wrapper stripping}, which removes outer wrapper logic, with
\emph{call-site transport}, which substitutes the helper at its call site while retaining preprocessing,
postprocessing, exceptions, and return flow. Admission requires null-threshold compliance and
preservation of each positive control's signature and positive distance; candidates and selection
rules were frozen before outcomes. We withdrew an earlier ``$8/8$'' cross-repository claim because
four pylint test IDs did not exist, so its corrected witness is post-diagnostic; requests/pytest is
separately pre-outcome held out. Wrapper stripping erases all $8/8$ corrected and $8/8$ held-out
failures, while call-site transport preserves them. A frozen renaming/inversion gate similarly finds
$10/16$ failures erased by destructive maps versus $16/16$ preserved by family-specific maps. On
$36$ natural states, call-site transport preserves every witness, reduces mean distance
$.1163\!\to\!.0244$ ($79.1\%$), and wins all four held-out folds with zero false invariance. Failed,
corrected, and held-out records are never pooled (App.~B.9--B.10).

\section{Related Work}
\label{sec:related}
Unknown-map identification, extremal alignment, and robust matching are established
\citep{gulrajani2022identifiability,lember2014optimal,morucci2022robust}; process conformance also
summarizes all optimal alignments \citep{baer2025alloptimal}. We borrow these tools but target the
sharp downstream diagnosis or compiled update after validating the cross-view map. Probabilistic or
soft matching instead declares a law or modified objective \citep{yousefi2019probpit}; known
tie-breaking impossibilities do not show when a hidden completion becomes a replicated temporal
policy \citep{feys2026tiebreaking}.

Non-singleton lower levels, normalization pathologies, and GRPO gradients are established
\citep{liu2020bilevel,chen2024hypergradient,masiha2026select,liu2026gdpo,fontana2026hidden,
zhou2026grpoustat}. Our source-faithful interface composes these classical convex, MAX-CUT, and
elimination tools \citep{fliege2000steepest,desideri2012mgda,sener2018multiobjective,
jaggi2013frankwolfe,garey1976simplified,koller2009pgm}. Unlike perturbation and metamorphic tests of a
declared statistic \citep{repomirage,mao2023equivariance,murphy2008metamorphic,dbpa2024}, we first
test whether correspondence identifies that statistic or update, using standard partial-identification
machinery \citep{kaido2019projection,scott2016cz}.

\section{Limitations and Conclusion}
The natural gradient audit is post hoc and uses one $2{,}560$-parameter tensor. Our worst-case
construction is synthetic; finite-source executions are conformance checks. We establish neither
final-policy degradation nor universal affectedness and do not know whether an undocumented
MatchTIR run enabled the intended branch. The repaired strict metric is not end-to-end accuracy; the
pre-registered downstream consequence remains \textsc{Weak}. Provider outputs lack immutable digests
and controllable seeds; we test one agent family, one checkpoint, and one compiler family (App.~C).

Within these boundaries, one returned optimum does not justify one conclusion. A correspondence
contract records the objective, exact-optimum set, completion, normalization, and readout, then over
all legal optima reports a point update, common-progress direction, or abstention. Hidden choices
change diagnosis and signed credit under a fixed objective, and exact repair moves a public ranking.

\clearpage
\appendix
\setcounter{figure}{0}
\setcounter{table}{0}
\setcounter{equation}{0}
\renewcommand{\thefigure}{A\arabic{figure}}
\renewcommand{\thetable}{A\arabic{table}}
\renewcommand{\theequation}{A\arabic{equation}}
\renewcommand{\theHfigure}{appendix.\arabic{figure}}
\renewcommand{\theHtable}{appendix.\arabic{table}}
\renewcommand{\theHequation}{appendix.\arabic{equation}}

\paragraph{Scope and navigation.}
Cross-references below use the numbering and labels of the main text. Appendix figures, tables, and
equations carry an A-prefix so that they cannot be confused with main-text Figure~1, Figure~2,
Table~1, or Equations~(1)--(4). Appendix~A restates and proves every formal main-text result.
Appendix~B gives additional experiments and robustness checks, with discovery, diagnostic,
held-out, and post-hoc evidence kept separate. Appendix~C states the exact claim boundaries, and
Appendix~D records the computational boundary of solver-independent auditing. The appendix is
therefore sufficient to check the formal claims without relying on implementation details, while
every empirical result retains its analysis unit, evidence status, and explicit nonclaim.

\noindent\begin{minipage}{\columnwidth}
\paragraph{Main-text claims and complete support.}
\begin{center}
\small
\setlength{\tabcolsep}{3pt}
\begin{tabular}{@{}p{.28\columnwidth}p{.64\columnwidth}@{}}
\toprule
Main-text claim & Complete support and boundary \\
\midrule
Thm.~1 and Eqs.~(1)--(3) & A.1--A.2 give the sharp identified set and complete proof of
one-sided insufficiency, witness control, and the exact linear feasibility boundary. The boundary is
linear; nonlinear maps are validated by declared controls. \\
Props.~1--2 and Eq.~(4) & A.4--A.6: normalization trilemma, randomization order, legal-update
trichotomy, and executable certificate; B.7 tests natural bodies. No final-policy guarantee is
inferred. \\
Thm.~2 and Table~1 & A.7: finite-bit reduction and tractable regimes; B.8: source-level
conformance; D: audit-cost boundary. \\
Natural results & Sec.~4.1/Fig.~2: B.1--B.2; Sec.~4.2: A.3 and B.3--B.6;
Sec.~4.3: B.11; Sec.~4.4: A.2 and B.9--B.10. Table~\ref{tab:evidenceledger}
records denominators, evidence status, and nonclaims. \\
\bottomrule
\end{tabular}
\end{center}
\end{minipage}

\noindent\begin{minipage}{\columnwidth}
\paragraph{Validity and scope map.}
\begin{center}
\small
\setlength{\tabcolsep}{3pt}
\begin{tabular}{@{}p{.38\columnwidth}p{.54\columnwidth}@{}}
\toprule
Potential concern & Decisive check and claim boundary \\
\midrule
Solver bug or numerical near-tie & A.1 and B.1--B.3 use exact-optimum sets; B.3 includes
objective-slack and all-distinct controls. \\
In-sample or backend-specific? & B.2 changes domain and backend; B.4 changes tasks and seeds and
reports every preregistered pass and failure. \\
Does random tie-breaking solve it? & A.5 distinguishes sampling from marginalization; B.5 measures
the resulting noncommutation in two independent samples. \\
Is Thm.~2 only a real-arithmetic sketch? & A.7 gives source realizability, polynomial-bit separation,
and a Karp reduction; B.8 and C delimit the natural-case claim. \\
Coordinate correction changes the headline? & B.1 preserves the frozen endpoint result ($44.6\%$)
beside the source-state result ($47.5\%$); both reverse the same two aggregate cells and the frozen
sweep spans $43.2$--$49.4\%$. \\
MatchTIR branch is default-off & B.3 reports the
reconstructed quick-start separately ($6/20$, $1/9$, $0/7$) and never pools it with the
intended-branch result. The claim is source-realizable non-identification, not deployment
prevalence. \\
Outcome-driven transport or score repair & B.9--B.10 retain the failed, post-diagnostic, and
pre-outcome held-out gates separately; B.11 derives its exact repair from the public strictness
contract before reading prediction contents. \\
No final-policy effect & B.11 supplies the measured operational consequence: nine pairwise model
comparisons and the top-ranked configuration of a public benchmark, surviving all $59$
leave-one-task-out omissions. Section~C explicitly excludes unmeasured training and policy effects. \\
\bottomrule
\end{tabular}
\end{center}
\end{minipage}

\noindent\begin{minipage}{\columnwidth}
\paragraph{Contribution attribution.}
The appendix does not relabel established machinery as new theory. A.1--A.2 use standard
partial-identification and principal-angle arguments; A.6 explicitly attributes minimum-norm
separation to classical multi-objective optimization. The claimed technical novelty is the
correspondence-induced composition: the compiler-dependent frontier and finite-bit construction in
Theorem~2, together with natural audits showing that hidden exact completions alter diagnoses,
signed credit, and a public ranking under fixed primary objectives.
\end{minipage}

\section{Proofs}

\paragraph{A.1\quad Identification relative to transport.}
Let $P$ denote the observed joint law of $(U,V)$ and let $\mathcal P_T$ be the admissible transport
class.
\begin{restated}[Main-text Equations (1)--(2): identification relative to transport]
The sharp population identified set is
$\mathcal I_T(P,\mathcal P_T)=\{\mathbb E_P[d(\Phi(V),U)]:\Phi\in\mathcal P_T\}$.
A singleton class point-identifies the response. If two admissible transports yield different
expectations, the observation law alone cannot point-identify the response associated with the
unknown transport. For compact $\mathcal P_T$ and a continuous functional, the extrema are attained.
\end{restated}
\begin{proof}
The observed law $P$ fixes the distribution of $(U,V)$ but does not select an element of
$\mathcal P_T$. For every admissible $\Phi$, a structural world that shares $P$ and uses $\Phi$ has
response $j(\Phi)=\mathbb E_P[d(\Phi(V),U)]$; conversely every data-compatible world allowed by the
model uses some $\Phi\in\mathcal P_T$. Hence the set of observationally compatible response values is
exactly the displayed image, proving sharpness. A singleton image is point-identified. If
$j(\Phi_0)\ne j(\Phi_1)$, the two worlds have the same $P$ but different target values, so no
measurable function of the observations can distinguish them. The unrestricted case is already
witnessed by deterministic $U=V=0$ on $\mathbb R$, $d(a,b)=|a-b|$, with
$\Phi_0(z)=z$ and $\Phi_1(z)=z+1$, which give responses $0$ and $1$. Attainment of extrema follows
from continuity on a compact class.
\end{proof}
The identified set need not itself be an interval for a nonconvex transport class. The paper reports
its infimum and supremum because threshold decisions depend only on those extrema; that interval is
the sharp interval hull. Randomized mixtures of admissible transports make the image an interval,
but are not required for the non-identification claim.

\emph{Correspondence-contract semantics.}
For objective $Q_x$ with exact optimizer set $\mathcal A_Q^\star(x)$ and readout $h$, the target is
indexed by policy: a declared selector $s$ defines $h(s(\mathcal A_Q^\star))$; a declared law $\pi$
defines $\mathbb E_{A\sim\pi}[h(A)]$; and a solver-invariant scientific claim is represented by the
set $\{h(A):A\in\mathcal A_Q^\star\}$. The first target is algorithmically well-defined without being
a solver-independent property of $x$. The third is sharp when every exact optimum is admissible and
no scientifically justified selector or law further restricts the contract. Replacing $Q$ by a soft
objective changes the estimand rather than selecting within the original hard-optimum set.

\paragraph{A.2\quad Two-sided correspondence validation.}
\begin{restatedtheorem}[Main-text Theorem 1: two-sided validity and its linear boundary]
Benign-only transport calibration does not certify response retention. If a transport is faithful
for an $L$-Lipschitz behavioral witness, transported distance controls witness change. In the linear
case, a nuisance-killing transport with positive response gain exists exactly when the nuisance and
response subspaces intersect only at zero; orthogonal projection has gain equal to the sine of their
smallest principal angle.
\end{restatedtheorem}
Part (c) takes the response subspace $\mathcal R$ to be nonzero. When
$\mathcal N=\{0\}$, we use the standard limiting convention
$\theta_{\min}(\mathcal R,\mathcal N)=\pi/2$, so the identity projection has gain one; a zero
response subspace is vacuous for the positive-response question.
\Cref{fig:alignmenttrapapp} isolates the transport-specific failure modes used by this certificate.

\begin{figure*}[t]
\centering
\begin{minipage}[t]{.315\textwidth}
\centering
\small (a) Raw comparison\par\smallskip
\includegraphics[width=\linewidth]{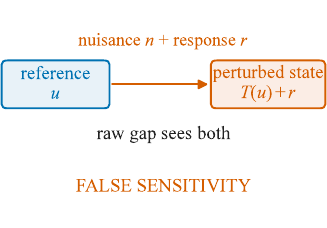}
\end{minipage}\hfill
\begin{minipage}[t]{.315\textwidth}
\centering
\small (b) Benign-only transport\par\smallskip
\includegraphics[width=\linewidth]{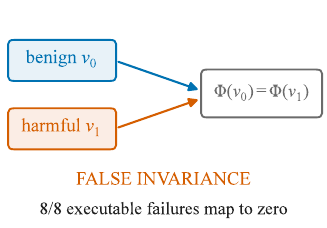}
\end{minipage}\hfill
\begin{minipage}[t]{.315\textwidth}
\centering
\small (c) Two-sided validation\par\smallskip
\includegraphics[width=\linewidth]{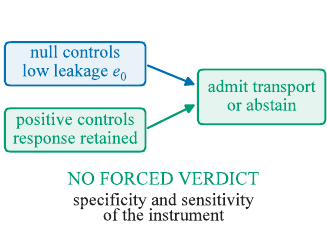}
\end{minipage}
\caption{Transport-specific failure modes underlying two-sided certification. (a)~Raw comparison
conflates nuisance and response. (b)~A map calibrated only on null controls can collapse harmful and
benign states. (c)~Two-sided validation separately tests nuisance leakage and response retention,
admitting a transport only when both criteria pass.}
\label{fig:alignmenttrapapp}
\end{figure*}

\begin{proof}
For one-sided insufficiency, fix any transport values on the null-control points. Let $v_0$ be a
calibrated target-view null point and let $v_1$ be an untested response point. Extend the transport to
$v_1$ by setting $\widetilde\Phi(v_1)=\widetilde\Phi(v_0)$, without changing it on any null control.
The null leakage is unchanged, while the response pair has zero transported distance. Choosing a
witness with $q_T(v_1)\ne q_T(v_0)$ also violates witness faithfulness. Hence no bound on null
leakage alone implies a positive response gain or faithful behavior.

For the witness certificate, faithfulness and Lipschitzness give directly
\begin{align*}
 |q_T(v_1)-q_T(v_0)|
 &= |q_0(\Phi(v_1))-q_0(\Phi(v_0))| \\
 &\le Ld(\Phi(v_1),\Phi(v_0)).
\end{align*}

For linear feasibility, suppose first that $0\ne r\in\mathcal R\cap\mathcal N$. Every linear map
$P$ that kills $\mathcal N$ has $Pr=0$, so its restricted response gain is zero. Conversely, if
$\mathcal R\cap\mathcal N=\{0\}$, orthogonal projection $P_{\mathcal N^\perp}$ is injective on
$\mathcal R$. The unit sphere in finite-dimensional $\mathcal R$ is compact, so the continuous map
$r\mapsto\|P_{\mathcal N^\perp}r\|$ attains a strictly positive minimum. By the definition of
principal angles, that minimum is $\sin\theta_{\min}(\mathcal R,\mathcal N)$. Thus a positive-gain
nuisance-killing map exists exactly under the stated intersection condition.
\end{proof}

\paragraph{A.3\quad Coordinatewise sign agreement as a diagnostic.}
\emph{Coordinatewise sign-agreement rule.}
Fix a canonical optimum $A_0\in\mathcal A^\star$ and a nonzero canonical credit coordinate
$a_i(A_0)$. Among coordinatewise rules that may only keep or zero this coordinate and that never keep
it when its sign or nonzero status changes over $\mathcal A^\star$, the unique pointwise-maximal rule
keeps it iff
\[
\{\operatorname{sign}a_i(A):A\in\mathcal A^\star\}
=\{\operatorname{sign}a_i(A_0)\}.
\]
\begin{proof}
If the displayed equality holds, keeping the coordinate satisfies the required correspondence-uniform
sign and nonzero condition. If it fails, some exact optimum gives either zero or a sign different from
the canonical sign, so every valid keep-or-zero rule must zero that coordinate. Applying this argument
independently to every coordinate proves feasibility and componentwise maximality. Uniqueness follows
because any other pointwise-maximal rule must keep every feasible coordinate and zero every infeasible
one.
\end{proof}
This rule concerns a declared set of exact optima. It neither chooses a probability law over
that set nor guarantees final-policy improvement. The retained coordinate vector need not be an
element of the legal update body, so the rule is a diagnostic sign certificate rather than an
automatically valid full-gradient update. A distributional soft assignment answers a different
question.

\paragraph{A.4\quad Relative-credit normalization tradeoffs and constructive bounds.}
Let $\mathbf1$ be the all-ones vector and $V=\mathbf1^\perp$. Current-peer isolation means
$\partial N_i/\partial x_j=0$ for every $j\ne i$.
\begin{restatedtheorem}[Main-text Proposition 2, normalization clauses, plus a symmetric-peer corollary]
Let $N$ be $C^1$ on a connected, translation-stable domain of nonconstant $n$-sample reward groups,
$n\ge3$.
\emph{(i)} If $N(x+c\mathbf1)=N(x)$ and $N$ is current-peer isolated, then $N$ is constant.
\emph{(ii)} If the restriction of $N$ to $V\setminus\{0\}$ is exactly scale neutral,
$N(ay)=N(y)$ for $a>0$, then $DN(ay)=a^{-1}DN(y)$. A nonconstant $N$ is not uniformly Lipschitz
near zero dispersion and has no continuous nonconstant extension to the constant-reward line.
\emph{(iii)} If $N$ is translation neutral and permutation equivariant, then at
$x=(u,v,\ldots,v)$, writing $\alpha=\partial N_1/\partial x_1$, every $j\ne1$ satisfies
$\partial N_1/\partial x_j=-\alpha/(n-1)$.
\end{restatedtheorem}
\begin{proof}
For (i), differentiate translation neutrality in $c$ to obtain $DN(x)\mathbf1=0$. Peer isolation makes
$DN(x)$ diagonal, so every diagonal entry is zero. Hence $DN=0$, and connectedness makes $N$
constant. For (ii), differentiating $N(ay)=N(y)$ with respect to $y$ gives
$aDN(ay)=DN(y)$. Since $V\setminus\{0\}$ is connected for $n\ge3$, a nonconstant $C^1$ map has a
nonzero derivative somewhere; its norm then diverges as $a^{-1}$ along that ray. A continuous
extension at zero would instead give $N(y)=\lim_{a\downarrow0}N(ay)=N(0)$ for every $y$, a
contradiction. For (iii), permutations of coordinates $2,\ldots,n$ fix $x$, so equivariance makes
all first-row peer derivatives equal to some $\beta$. Origin neutrality gives
$\alpha+(n-1)\beta=0$.
\end{proof}

These tradeoffs are class-level: inverse-dispersion instability and peer influence are not peculiar to
the standard-deviation formula. It also separates three repairs that are sometimes conflated. Let
$P=I-\mathbf1\mathbf1^\top/n$, $y=Px$, and
\[
N_\tau(x)=\frac{y}{\sqrt{\|y\|^2/(n-1)+\tau^2}},\qquad \tau>0.
\]
Its Jacobian on $V$ has eigenvalue $1/s$ orthogonal to $y$ and $\tau^2/s^3$ along $y$, where
$s^2=\|y\|^2/(n-1)+\tau^2$; the common direction is annihilated. Therefore
$\|DN_\tau(x)\|_{2\to2}\le1/\tau$ globally. Damping preserves translation neutrality and caps magnitude
amplification, but deliberately gives up exact scale neutrality near ties and does not certify a
sign. By contrast, $(x_i-b)/\sigma$ with an independently estimated $b$ and $\sigma>0$ is
$1/\sigma$-Lipschitz and has zero current-peer derivatives, but replaces within-group origin and
scale neutrality by covariance with a declared external anchor. Finally, the all-optima mask in
A.3 addresses the upstream set-valued fiber: it retains only signs shared by every exact optimum.
These guarantees are complementary, not claims that one normalizer is universally optimal.

\paragraph{A.5\quad Randomization order and the credit estimand.}
Let $r(A)$ be the reward vector induced by exact optimum $A\in\mathcal A^\star$, let $\pi$ be a
declared law over exact optima, and let $N$ denote the downstream relative normalizer acting on the
joint reward vector.
\begin{restated}[Main-text Proposition 2, randomization clause: normalization noncommutation]
The two randomized-credit targets
\[
 \begin{aligned}
 C_{\mathrm{ENA}}&=\mathbb E_{A\sim\pi}[N(r(A))],\\
 C_{\mathrm{NAE}}&=N\!\left(\mathbb E_{A\sim\pi}[r(A)]\right)
 \end{aligned}
\]
need not agree. They agree for every finitely supported law on the reward fiber if and only if $N$
is affine on its convex hull. Hence a random tie-breaking law alone does not identify credit unless
the compiler also declares whether matching is sampled before normalization or marginalized before
normalization.
\end{restated}
\begin{proof}
If $N$ is affine, it preserves every finite convex combination, giving equality. Conversely, suppose
the equality holds for every finitely supported $\pi$. Every point in the convex hull of the reward
fiber is a finite convex combination of fiber points, and the displayed equality says that $N$ maps
each such combination to the same combination of its images. This is precisely affinity on that
convex hull. Sampling an optimum and then normalizing computes $C_{\mathrm{ENA}}$; averaging rewards
over the fiber and then normalizing computes $C_{\mathrm{NAE}}$. Without affinity or a declared
order, the two are distinct estimands.
\end{proof}
Standard centering followed by division by a sample-dependent scale is non-affine. Equality can still
hold accidentally for a particular fiber or law; the proposition says it cannot be presumed from
randomization itself. This is an application of the elementary preservation-of-mixtures criterion,
not a claim of a new general theorem about nonlinear expectations.

\paragraph{A.6\quad Classical separation applied to the legal update body.}
\label{app:updatebody}
For a fixed parameter value $\theta$, write
$\mathcal U_\theta=\{g_\theta(A):A\in\mathcal F\}$ and
$\mathcal K_\theta=\operatorname{conv}(\mathcal U_\theta)$. The exact fiber is finite in the audited
pipelines; the same statements hold for compact fibers and continuous $g_\theta$.
The following is the classical minimum-norm separation characterization applied to this
correspondence-induced body \citep{fliege2000steepest,desideri2012mgda,sener2018multiobjective}.
\begin{restatedtheorem}[Main-text Proposition 1: solver-independent update trichotomy]
\emph{(i)} The update is identified iff $\mathcal K_\theta$ is a singleton.
\emph{(ii)} A unit vector $u$ and margin $m>0$ satisfying
$\inf_{g\in\mathcal K_\theta}\langle u,g\rangle\ge m$ exist iff
$0\notin\mathcal K_\theta$. If
$x^\star=\arg\min_{x\in\mathcal K_\theta}\|x\|$, then
$u^\star=x^\star/\|x^\star\|$ has margin at least $\|x^\star\|$.
\emph{(iii)} If $0\in\mathcal K_\theta$, no strict common-progress direction exists.
\end{restatedtheorem}
\begin{proof}
Part (i) is the definition of solver-invariant point identification. If $0\notin\mathcal K_\theta$,
compact convex separation gives a strict separating direction. More explicitly, the projection
optimality condition for $x^\star$ is
$\langle g-x^\star,x^\star\rangle\ge0$ for every
$g\in\mathcal K_\theta$. Dividing by $\|x^\star\|$ yields
$\langle u^\star,g\rangle\ge\|x^\star\|>0$. Conversely, if
$0=\sum_j\lambda_jg_j\in\mathcal K_\theta$ and one $u$ had
$\langle u,g\rangle>0$ for every legal $g$, then
$0=\langle u,\sum_j\lambda_jg_j\rangle>0$, a contradiction.
\end{proof}

If every legal compiled objective $F_A$ is $L$-smooth at $\theta$ and
$\nabla F_A(\theta)=g_\theta(A)$, the ascent lemma gives
\[
F_A(\theta+\eta u)\ge F_A(\theta)+
\eta\langle u,g_\theta(A)\rangle-\frac{L\eta^2}{2}.
\]
Thus a certified margin $m$ improves every legal objective for
$0<\eta<2m/L$. This is a local common-improvement statement, not a claim about an entire training
trajectory.

The closest point can be approximated from a linear minimization oracle. At iterate $x\ne0$, let
$s(x)\in\arg\min_{g\in\mathcal K_\theta}\langle x,g\rangle$ and define the
Frank--Wolfe gap
$\zeta=\langle x,x-s(x)\rangle$. Then
\[
\min_{g\in\mathcal K_\theta}
\left\langle\frac{x}{\|x\|},g\right\rangle
=\frac{\|x\|^2-\zeta}{\|x\|}
=\|x\|-\frac{\zeta}{\|x\|}.
\]
Consequently, $\zeta<\|x\|^2$ is an executable strict-progress certificate at any iteration.
Fully corrective Frank--Wolfe/Gilbert optimization over the active oracle vertices is the standard
procedure used in our implementation \citep{gilbert1966minnorm,jaggi2013frankwolfe}; no new
convergence rate is claimed. Each iteration adds one support-oracle vertex and then reoptimizes the
weights over the active hull. Because the audited natural gradients have norm near $10^{-3}$, the
points are first divided by their median norm; the run then stops when the Frank--Wolfe gap
satisfies $\zeta\le10^{-8}$ on that rescaled problem, with an iteration cap of $1{,}000$. All
$15$ audited task bodies met the gap criterion, using at most $9$ iterations, so no reported
dispatch decision was truncated by the cap.

\paragraph{A.7\quad Compiler tractability frontier.}
\label{app:compilerfrontier}
This section gives the complete proof of main-text Theorem~2 (the main text) and the
tractability statements summarized in main-text Table~1 (p.~5). It first establishes the affine and
bounded-interaction regimes, then gives a source-realizable finite-bit reduction for shared
standardization. The reduction is a worst-case compiler result; B.8 and C state the implementation
checks and natural-case limits separately.

\begin{restatedtheorem}[Main-text Theorem 2: group standardization is an Ising compiler]
Fix rational discount $\gamma=9/10$ and a positive rational stabilizer $\epsilon$. For every nonempty
unweighted graph $G=(V,E)$ with $m=|V|$, there is a polynomial-size family of literal tool-call
similarity instances such that each ambiguous assignment fiber has exactly two exact optima, both
preserve the same scalar trajectory reward, every trajectory total is equal, only two deterministic
anchor trajectories carry the queried policy-score derivative, and the compiled gradient coordinate
is
\[
\begin{aligned}
g_G(y)&=C_G+J_G\sum_{(i,j)\in E}y_i y_j,\\
y_i&\in\{-1,+1\},\qquad J_G>0.
\end{aligned}
\]
Adjacent cut sizes stay $\Omega(m^{-14})$ apart in support. On the corresponding finite-bit promise
formulation, inverse-polynomial-additive lower support is NP-hard and universal all-optima threshold
certification is coNP-hard.
\end{restatedtheorem}
\emph{Affine exact-face formulation.}
Let fiber $i$ be the set of integral exact optima of a bipartite assignment problem. Its convex hull
$P_i$ has the usual row/column assignment constraints plus the equality fixing the optimal objective
value. Assignment integrality implies that the vertices of $P_i$ are exactly the integral exact
optima. If the compiled update is affine in the product variable
$z=(z_1,\ldots,z_k)\in P_1\times\cdots\times P_k$,
\[
g(z)=c+Mz,
\]
then the legal update body is exactly $g(P_1\times\cdots\times P_k)$. Directional support is an LP
over this compact extended formulation, and the closest point to the origin is a convex QP (or its
epigraph SOCP). If every $P_i$ is a segment with endpoints $p_i^-,p_i^+$, then
\[
\mathcal K=c'+\sum_i[-v_i,v_i],\qquad
v_i=\frac{M(p_i^+-p_i^-)}2,
\]
a translated zonotope. These are standard polyhedral facts; the contribution is identifying the
exact correspondence face as the uncertainty source consumed by the reward compiler.

\emph{Local nonlinear factorization.}
Let finite fiber state $a_i$ range over $\mathcal A_i$, $|\mathcal A_i|\le q$. If turn $t$ depends
only on states $a_{S_t}$, any fixed projected compiler query has the exact form
\[
g(a_1,\ldots,a_k)=c+\sum_t\psi_t(a_{S_t}).
\]
Join two variables when they co-occur in a factor. Given a tree decomposition of width $w$, assign
each factor to a bag containing its scope and eliminate variables by the standard min-sum recursion.
Every intermediate table has at most $q^{w+1}$ entries, giving time
$\operatorname{poly}(n)q^{w+1}$ and memory
$\operatorname{poly}(n)q^w$; replacing min by max gives upper support. This is standard graphical
model variable elimination \citep{koller2009pgm}. Backpointers recover the minimizing exact
correspondence required by the support-oracle method.

\begin{proof}[Proof of Main-text Theorem 2.]
\emph{Restricted Ising representation.}
We now prove the non-affine converse used in the main-text Ising compiler theorem. Let
$G=(V,E)$ be a nonempty unweighted graph. Isolated vertices can be deleted. Write
$m=|V|$, $M=|E|$, $T=M+1$, and $N=m+2$. There is one ambiguous trajectory per vertex and two
deterministic anchor trajectories.

Index the first $M$ turns by edges. For vertex $i$, prescribe desired discounted returns
$d^+_{i,e}=1$ when $e$ is incident to $i$ and $0$ otherwise; define $d^-_i=-d^+_i$.
The final desired return is chosen so that the undiscounted raw-reward total is zero. In detail, for
$0\le t<T-1$ set
$r_t=d_t-\gamma d_{t+1}$ and $r_{T-1}=d_{T-1}$. Since
\[
\sum_{t=0}^{T-1}r_t
=d_0+(1-\gamma)\sum_{t=1}^{T-1}d_t,
\]
setting
$d_{T-1}=-[d_0+(1-\gamma)\sum_{t=1}^{T-2}d_t]/(1-\gamma)$
makes this total zero. The positive anchor uses desired return $+1$ at every edge turn and the same
compensation; the negative anchor is its negation.

For one zero-total rational raw vector $r$, take common denominator $D=10$ and integers
$z_t=Dr_t$. Let $R=\max_t|z_t|$ and $C_0=(T-1)R+1$. Construct a $T\times T$ agreement-count matrix
with $C_0+z_t$ on identity edge $(t,t)$, $C_0-z_t$ on cyclic edge
$(t,t+1\bmod T)$, and zero elsewhere. Give every predicted and reference call the same tool name and
the same $L=2TC_0$ parameter keys. A key dedicated to pair $(i,j)$ has equal values only for that
pair; all remaining values are distinct. The released similarity is therefore
\[
s_{ij}=\frac{2+c_{ij}}{L+2}.
\]
Identity and cyclic shift each collect bonus $TC_0$. Any other full permutation uses at most
$T-1$ cycle edges and has bonus at most
$(T-1)(C_0+R)<TC_0$. Since every edge has a positive baseline, every partial assignment can be
strictly improved by completion. Hence the exact fiber has precisely the two designated optima.
Their raw reward vectors are $K+\delta r$ and $K-\delta r$, where
$K=(2+C_0)/(L+2)$ and $\delta=D/(L+2)$, so both have scalar total $TK$.
For each anchor, place its strictly positive agreement counts only on the diagonal. Every full
matching receives the same positive name-match baseline, while only the diagonal receives all
agreement bonuses; the full diagonal is therefore the unique optimum and has the same total $TK$.

The common raw offset produces only a turn-dependent offset shared by all trajectories and vanishes
under group centering. At edge $e=(i,j)$, the variable part of the discounted-return group is
\[
\delta(y_i,y_j,0,\ldots,0,+1,-1),\qquad y_i\in\{-1,+1\}.
\]
Its unscaled sample variance is
\[
v_e(y)=\frac{4-(y_i+y_j)^2/N}{N-1}.
\]
Thus $v_{\rm cut}=4/(N-1)$ and $v_{\rm uncut}=4/N$. Give the two anchors projected policy-score
derivatives $+1$ and $-1$, and all other trajectories derivative zero. After the compiler's equal
global/local mixture, the queried edge contribution is
\[
\alpha(v)=\frac{\delta}{\delta\sqrt v+\epsilon}.
\]
All trajectory totals are equal, so the global branch is zero. Since
$\alpha_{\rm cut}<\alpha_{\rm uncut}$,
\[
\begin{aligned}
g_G(y)
&=\alpha_{\rm cut}\operatorname{cut}_G(y)
  +\alpha_{\rm uncut}\bigl(M-\operatorname{cut}_G(y)\bigr)\\
&=C_G+J_G\sum_{(i,j)\in E}y_i y_j
\end{aligned}
\]
for $J_G=(\alpha_{\rm uncut}-\alpha_{\rm cut})/2>0$. Its minimizers are exactly maximum cuts.

The score pattern is itself realizable. At $\theta=0$, use a two-action softmax with realized and
alternative features $(+1,-1)$ in the positive-anchor contexts, swap them in the negative-anchor
contexts, and use $(0,0)$ elsewhere. The derivatives of the realized log probabilities are then
$+1,-1,0$ under one shared scalar parameter.

Finally, the reduction has polynomial bit complexity. Here $T=O(m^2)$,
$R=O(M)$, $C_0=O(M^2)$, $L=O(M^3)=O(m^6)$, and
$\delta=\Omega(m^{-6})$. Moreover,
\[
\frac2{\sqrt{m+1}}-\frac2{\sqrt{m+2}}
\ge (m+2)^{-3/2},
\]
so $\Delta=\alpha_{\rm uncut}-\alpha_{\rm cut}=\Omega(m^{-14})$ for fixed positive rational
$\epsilon$. The lower support is
$M\alpha_{\rm uncut}-\Delta\operatorname{MaxCut}(G)$. Given a MAX-CUT threshold $k$, define
the midpoint
$\tau_k=M\alpha_{\rm uncut}-(k-\tfrac12)\Delta$. If
$\operatorname{MaxCut}(G)\ge k$, lower support is at most $\tau_k-\Delta/2$; if
$\operatorname{MaxCut}(G)\le k-1$, it is at least $\tau_k+\Delta/2$.
Square roots of polynomial-bit rationals admit polynomial-time dyadic enclosure, so approximating
$\tau_k$ within $\Delta/8$ requires only polynomially many bits and preserves both promise sides.
This is a finite-bit Karp reduction from
unweighted MAX-CUT \citep{garey1976simplified}. Lower-support gap decision is NP-hard; a violating
correspondence is an NP-hard search problem; and the universal threshold statement is coNP-hard.
We claim neither completeness nor fixed-precision asymptotic hardness.
\end{proof}

\section{Additional Experimental Evidence and Robustness}

Each analysis below extends a main-text claim with an additional derivation, robustness check, or
scope test. \Cref{tab:evidenceledger} separates analysis units, evidence status, supported claims,
and explicit nonclaims. This prevents discovery, post-diagnostic, held-out, and post-hoc results from
being pooled under one evidential label.

\begin{table*}[t]
\centering
{\small
\setlength{\tabcolsep}{3pt}
\begin{tabular}{@{}p{.15\textwidth}p{.14\textwidth}p{.13\textwidth}p{.22\textwidth}p{.28\textwidth}@{}}
\toprule
Evidence & Analysis unit & Provenance & Supports & Does not support \\
\midrule
Temporal localization & $1{,}586$ nonzero code/SQL pairs & frozen thresholds and exact reanalysis &
$47.5\%$ lack one early/late label across all optima & first discovery of the source paper's routing
hazard; a claim about every trajectory metric \\
Strict trajectory repair & $1{,}532$ public task-configuration cells & pre-registered source audit &
$14.23\%$ false-perfect; nine model-pair reversals & a new sequence metric; changed end-to-end
accuracy; universal benchmark affectedness \\
MatchTIR discovery & $20$ materially affected task groups & frozen $800$-rollout sample &
$14/20$ groups admit intended-credit sign reversal & branch use in the released training run; final-policy damage \\
MatchTIR quick-start & $20$ material groups; $7$ clean distinct groups & reconstructed public path &
$6/20$ reversals; $0/7$ in the clean subset & use of the intended branch; final-policy effect \\
MatchTIR replication & $22$ materially affected task groups & task-and-seed-disjoint frozen sample &
$14/22$ groups replicate sign reversal & deployment prevalence outside the sampled frame \\
Natural update bodies & $15$ frozen tasks, one parameter tensor & post-hoc certificate audit &
$14$ common-progress certificates and one abstention & changed training decisions or policy improvement \\
Prospective transport & $24$ paired trajectories, $72$ stages & pre-registered &
agent-free residue diagnosis and a failed natural-response gate & successful prospective response preservation \\
Corrected witnesses & eight harmful controls & post-diagnostic frozen replication &
benign-perfect destructive transport can erase all eight & independent discovery evidence \\
Held-out transport & two benign and eight harmful controls & pre-outcome held-out repositories &
frozen two-sided selection preserves all failures & natural-agent or second-family generality \\
\bottomrule
\end{tabular}
}
\caption{Evidence map for the current manuscript. ``Frozen'' denotes a fixed analysis or protocol;
``pre-registered'' fixes its population, analysis, and decision criteria before outcomes;
``held out'' additionally means the listed tasks, seeds, repositories, or outcomes were not used to
construct the audited rule. Post-diagnostic and post-hoc results are labeled and are not treated as
independent confirmation. Fractions are descriptive for the stated denominator.}
\label{tab:evidenceledger}
\end{table*}
\paragraph{B.1\quad Sharp temporal localization over all optimal alignments.}
We downloaded the public normalized trajectories associated with
\citet{consistencytestable2026} and retained every released paired SWE-bench or Spider2-DBT record. We
exactly replicate the released action extractor: every nonempty tool name emitted by an agent step
becomes one token. No tool is renamed,
collapsed, or semantically matched in the primary analysis.

For sequences $a_{1:n},b_{1:m}$, let $D_{ij}$ be the suffix unit-cost Levenshtein optimum. An edge
$e:(i,j)\to(i',j')$ is admissible exactly when its local cost plus $D_{i'j'}$ equals $D_{ij}$; these
edges form the optimal-alignment DAG. For each edit edge we record source position $i/n$ and
symmetric DP-state position $\tfrac12(i/n+j/m)$. Every optimal path contains exactly $D_{00}$ edits,
so forward addition of edge positions with min/max semirings gives the sharp minimum and maximum
mean position in $O(nm)$ time and memory. Lowest/highest optimal paths are known sequence-alignment
objects \citep{lember2014optimal}; our use is to bound an agent-evaluation mechanism claim, not to
claim a new alignment algorithm.

The analysis and cutoffs were frozen before the natural-data run. Exact-null and unique-token
early/late anchors pass. We also exhaustively enumerate every alignment for all $961$ binary sequence
pairs of lengths zero through four; distance, optimal-path count, and all four extrema match the DAG
program. This exhaustive finite-sequence validation was completed before analysis. Across $1{,}586$ non-null
pairs, $47.5\%$ are alignment-unidentified under both normalizations. Left/right-priority optimal
tracebacks differ in label for $55.9\%$. Dominant-token share predicts symmetric range width
($\rho=.606$, pair-bootstrap $95\%$ CI $[.561,.647]$), and the natural heterogeneous-action control
is narrower (OpenHands $.094$ versus Codex $.466$). Of the thirteen audited cells the source study
attaches a directional early/late mechanism claim to three (two Codex, one OpenHands); both Codex
claims reverse over equally optimal alignments and the heterogeneous-action OpenHands claim does not,
consistent with the repeated-action mechanism. Collapsing all tokens, a response-destructive
diagnostic excluded from evidence, raises ambiguity from $48.8\%$ (on the $1{,}545$ collapse-comparable
pairs) to $93.0\%$.

We disclose one coordinate-convention correction. The frozen first analysis placed an edit at its DP
transition endpoint. Comparison with the source paper's reported $.217$ position then exposed
that the paper defines local weight at state $(i,j)$. We preserved v1 and reran after changing only
that coordinate convention: endpoint v1 gives $44.6\%$ ambiguity and the same two reversing cells;
source-state v2 gives $47.5\%$. The preregistered threshold sweep gives
$43.2\%$, $47.5\%$, and $49.4\%$ ambiguity at early/late cutoffs $(.35,.45)$, $(.40,.50)$, and
$(.45,.55)$, with the same two aggregate reversals. Because v2 follows a disclosed post-v1
correction, we report both rather than describing v2 alone as pristine preregistered confirmation.

\begin{table*}[t]
\centering
{\scriptsize
\setlength{\tabcolsep}{4pt}
\begin{tabular}{@{}p{.23\textwidth}p{.10\textwidth}p{.19\textwidth}p{.14\textwidth}p{.25\textwidth}@{}}
\toprule
Cell & Nonzero pairs & Sharp source-state mean range & Pairwise unidentified & Published directional claim \& audit consequence \\
\midrule
Spider2-DBT / Codex / header shuffle & 64 & $[.211,.816]$ & $79.7\%$ & early; reversible over exact optima \\
Spider2-DBT / Codex / header translation & 64 & $[.229,.774]$ & $78.1\%$ & none \\
Spider2-DBT / Codex / timestamp & 64 & $[.208,.827]$ & $82.8\%$ & none \\
Spider2-DBT / OpenHands / header shuffle & 64 & $[.510,.682]$ & $14.1\%$ & late; retained over exact optima \\
Spider2-DBT / OpenHands / header translation & 64 & $[.518,.685]$ & $12.5\%$ & none \\
Spider2-DBT / OpenHands / timestamp & 63 & $[.525,.702]$ & $12.7\%$ & none \\
SWE-bench / Codex / injection & 389 & $[.266,.687]$ & $56.6\%$ & none \\
SWE-bench / Codex / linear MCP & 499 & $[.298,.747]$ & $70.1\%$ & early; reversible over exact optima \\
SWE-bench / OpenHands / injection & 57 & $[.472,.517]$ & $0.0\%$ & none \\
SWE-bench / OpenHands / linear MCP & 8 & $[.415,.463]$ & $0.0\%$ & none \\
SWE-bench / OpenHands / noise & 161 & $[.457,.505]$ & $1.2\%$ & none \\
SWE-bench / OpenHands / paraphrase & 31 & $[.483,.531]$ & $0.0\%$ & none \\
SWE-bench / OpenHands / translation & 58 & $[.471,.531]$ & $5.2\%$ & none \\
\bottomrule
\end{tabular}
}
\caption{Complete cell-level ledger for the temporal-localization audit under the corrected
source-state convention. A range is the interval between the aggregate mean of the per-pair sharp
lower endpoints and the aggregate mean of the per-pair sharp upper endpoints. ``Pairwise
unidentified'' is the fraction whose sharp range crosses the frozen early/late decision under both
reported normalizations. Codex denotes Codex-GPT5mini and OpenHands denotes OpenHands-KimiK2.}
\label{tab:temporalcellledger}
\end{table*}
\paragraph{B.2\quad Second natural audit: global matching (M3-Bench).} A second public pipeline
exhibits the same non-identification under a \emph{different} alignment family (global Hungarian
matching, not an edit-distance DAG) and \emph{domain} (multi-modal medical/tool-use, not code/SQL).
M3-Bench \citep{m3bench2025} buckets tool calls by name and maximizes summed argument-embedding
cosine by Hungarian assignment, then reads Order Consistency, Step Coherence, and Merge Purity off the
chosen assignment. We audit the pinned paper-release ground truth and dataset revision without
generating model output. Whenever the same tool name and canonicalized arguments recur across steps, every permutation
of those copies keeps each edge at cosine one, so the assignment is a \emph{certified} global tie, not
an embedding near-tie. Enumerating these exact-tie faces and copying the upstream metric definitions,
all $260$ task/assignment combinations match the pinned official evaluator (max absolute difference:
Order Consistency $0$, Step Coherence $0$, Merge Purity $3.79\!\times\!10^{-8}$, the expected
\texttt{float32} effect). Of $213$ released rows ($208$ unique task IDs under a disclosed last-row
rule), $14$ carry a cross-step exact tie and a non-point structural score. Task \texttt{00130004}
(\texttt{medical}) has $32$ certified optima; compared to \emph{itself}, Order Consistency ranges the
full $[0,1]$, Step Coherence $[.43,1]$, Merge Purity $[.31,1]$---a hidden tie-break can diagnose the
identical trajectory as perfectly ordered or maximally inconsistent. Repeated identical calls are
ordinary agent behavior (iterated slide or figure generation, repeated retrieval or query), not a
pathology. The metric is non-identified when ordering depends most directly on these repeated calls.
Self-comparison is an axiomatic
positive control (no model, judge, or sampling noise can explain the range); a \emph{non-identical}
cross-revision pair for task \texttt{00170002} (paper release nine calls, current release eight) has
four certified optima with Step Coherence $[.44,.89]$, Order Consistency $[.67,.82]$, Merge Purity
$[.61,.89]$. The ambiguity persists across both public dataset revisions ($14$ tie-tasks in the paper
release, $10$ in the current release), so it is a standing structural property, not a fixed defect:
the specific tie-tasks turn over (\texttt{00130004} is paper-release-only; all $10$ current ties are
\texttt{generatepowerpoint}) but the repeated-call mechanism survives the churn. Because the
repository releases no per-model predictions, this is a metric-\emph{identification} failure, not a
demonstrated model-ranking reversal. The multiplicity lives in the fixed reference, so any prediction
inherits these ties even without repeating a call: routing a single predicted call into a reference
bucket that already holds identical copies suffices, and self-matching is the confound-free witness
of a defect every prediction against these tie-tasks incurs. \emph{Algorithmic scope.} Our two auditors share an interface but are not a
universal solver: additive statistics on the Levenshtein optimal-path DAG and linear functionals over
an explicit Hungarian optimal face admit exact sharp bounds; a general nonlinear statistic (e.g.
inversion-based Order Consistency) may require enumeration or certified relaxations.

\paragraph{B.3\quad Natural learning audit: exact-optimum MatchTIR credit.}
MatchTIR \citep{qu2026matchtir} derives dense tool-call rewards by bipartite matching and combines
turn-level with trajectory-level advantage. We audit the released Qwen3-4B MatchTIR checkpoint,
public training data, and released source version. The checkpoint supplies natural post-training rollouts on the distribution
where the process reward is defined; this is not a held-out model comparison.

\emph{Study design.} Before generation, we froze source/data versions, $200$ task rows
(half with repeated ground-truth tool names), four rollout seeds, temperature/top-$p$ $=1$, six
turns, and $768$ new tokens per turn. All $800/800$ trajectories completed; $20$ contain a tool error
and parser attrition is zero. Outcome-blind memory-cap resumes changed no scientific setting.

\emph{Exact correspondence audit.} For each trajectory with two to eight parsed calls, we reproduce
the released sorted-edge greedy assignment, compute the exact maximum-weight partial one-to-one
objective, and enumerate every unique per-call reward vector attained by an exact optimum. A bitmask
dynamic program is exponential only in the frozen at-most-eight reference calls; all retained groups
are exactly enumerated, with no vector truncation or Cartesian-product exclusion. A trajectory is
pre-registered as materially ambiguous when it has at least two distinct predicted-call JSONs and an
exact-optimum per-call reward width at least $.10$. Row-clustered uncertainty uses $10{,}000$
preregistered resamples of the $200$ task rows.

\emph{Downstream credit.} In each four-rollout task group, we reconstruct the intended $.9$-discounted
turn returns and dual-level standardized advantage, enumerate the Cartesian product of exact-optimal
reward vectors, and record every coordinate sign. The frozen gate required a prevalence-CI lower
bound of $5\%$, at least $20$ material groups and $20\%$ sign-flip groups, at least $60\%$ retained
canonical nonzero coordinates, and coverage of at least two source strata.

All four criteria pass. Of $495$ eligible rollouts, $93$ are material
($18.8\%$, row-clustered $95\%$ CI $[14.1,24.0]\%$), spanning
\texttt{multi\_hop}, \texttt{parallel\_multi\_hop}, and
\texttt{parallel\_single\_hop}. The released greedy objective is never suboptimal. Of $20$ material
four-rollout groups, $14$ admit an intended advantage-sign reversal. Across the $81$ exactly
enumerated groups, the strict mask retains $759/864=87.8\%$ canonical nonzero coordinates. Restricted
to material groups it retains $332/437=76.0\%$, versus $144/437=33.0\%$ under whole-group abstention;
within the $14$ sign-flip groups it retains $188/293=64.2\%$, versus zero under whole-group
abstention.

\emph{Tie-breaking sensitivity.} Under a post-hoc uniform law over distinct exact-optimal reward
vectors, the discovery and replication samples give per-coordinate sign-change probabilities
$.201/.203$, mean task-level risks $.547/.494$, and nonzero risk in $14/20$ and $14/22$ groups. Thus
the effect does not require an adversarial optimum, but these law-conditional quantities do not
replace the distribution-free all-optima certificate.

\emph{Near-optimal sensitivity.} We preregistered the relative-slack set
$W^\star-W(M)\leq\rho\max\{|W^\star|,1\}$. Forced residual assignments give sharp coordinatewise
ranges. Zero slack exactly reproduces the headline counts; no new material trajectory appears through
$\rho=.01$ in either sample. At $.05$, counts rise from $93$ to $107$ and from $90$ to $116$; the
median slack required to create a $.10$ range is $.20$ in both samples. The claim is therefore stable
to small solver tolerances. Soft matching remains a different estimand, while uniform tie-breaking
requires a declared law.

\emph{Post-hoc hostile checks.} Removing tool errors leaves $84/476=17.6\%$ material rollouts;
requiring distinct predicted-call JSONs leaves $29/408=7.1\%$; imposing both leaves $24/394=6.1\%$.
When only the nine clean all-distinct material rollouts vary, $5/7$ affected groups still reverse sign
and the mask retains $150/159=94.3\%$ canonical updates. Duplicates amplify but do not create the
result.

\emph{Public-configuration boundary.} The released quick-start invokes the multistep process scorer
but does not visibly enable the default-off multi-turn mask. Reconstructing
that effective sequence-mask path gives $6/20$ sign-flip groups over all material drivers,
$1/9$ for all-distinct drivers, and $0/7$ after additionally removing tool errors. We therefore make
no claim about an undocumented author override, final-policy degradation, or the reported benchmark
scores. The primary result is about the turn-level mechanism described by the paper; the
quick-start reconstruction is a disclosed boundary, not pooled with it.

\paragraph{B.4\quad Task-and-seed-held-out credit replication.}
\label{app:heldoutcredit}
The natural audit above establishes that the intended advantage can depend on a rollout's own tied
matching. We next asked a stricter question: can ambiguity in one rollout change the sign assigned to
a different rollout whose own local reward is fixed? This is a nonlocal consequence of the group
baseline, not a second matching ambiguity at the target.

\emph{Exact centering law.}
Consider one nondegenerate normalization cell containing $n$ rollouts. Let rollout $i$'s scalar local
reward range over $R_i=[\ell_i,u_i]$, with choices independent across rollouts because each trace is
matched separately, and let
$C_i=x_i-n^{-1}\sum_{k=1}^n x_k$ be its centered credit.
Then its sharp identified interval is
\begin{equation}
  \label{eq:peercontagion}
  \begin{aligned}
  C_i&\in[L_i,U_i],\\
  L_i&=\frac{(n-1)\ell_i-\sum_{j\ne i}u_j}{n},\\
  U_i&=\frac{(n-1)u_i-\sum_{j\ne i}\ell_j}{n}.
  \end{aligned}
\end{equation}
Writing $w_i=u_i-\ell_i$, the width is exactly
\begin{equation}
  \label{eq:peerwidth}
  U_i-L_i=\frac{(n-1)w_i+\sum_{j\ne i}w_j}{n}.
\end{equation}
Both bounds are attained by choosing the displayed endpoint in every independently matched rollout.
Consequently, even if the target is locally identified ($w_i=0$), peer $j$ contributes exactly
$w_j/n$ to its centered-credit uncertainty. Standardization divides $C_i$ by a positive empirical
scale, so its sign is unchanged in a nondegenerate cell. Hence $L_i>0$ and $U_i<0$ certify positive
and negative signs respectively; otherwise the rule abstains. The algebra is elementary. Its role is
to expose why locally identified matching credit is not compositional under a coupled baseline and to
provide a certificate that avoids Cartesian enumeration.

\emph{Frozen replication.}
Before inspecting a held-out outcome, we selected $200$ task rows from the $1{,}440$ eligible rows
unused above, exactly matching the original repeated-name and source-stratum counts. We replaced all
four rollout seeds with fresh values disjoint from discovery and retained the same checkpoint version, temperature,
top-$p$, six-turn limit, and token limit. Thus neither task nor seed overlaps the first audit; the
rows are not held out from checkpoint training. Six decision criteria were frozen before
generation. Completion ($800/800$), material/sign-flip groups ($22/14$), direct contagion
($7/14=.500$, Wilson $95\%$ CI $[.268,.732]$), and zero-false-sign recall
($283/402=.704$) pass. Margin coverage ($5/11=.455$) and the frozen risk score (AUC $.589$ versus
the $.70$ threshold) fail, so the preregistered verdict is \textsc{Weak}, not \textsc{Promote}. The
direct-contagion rate replicates the first audit's $6/14=.429$; the risk ranking does not and is not
refit.

\emph{Conditional mechanistic audit.}
Because C2 passed, we executed the separately frozen gradient audit on all 11 direct witnesses. To
avoid treating a tied input/output embedding as an output-only parameter, the audit differentiates
the checkpoint's $2{,}560$-dimensional final-RMSNorm scale under the released
sequence-mean/token-mean aggregation. All $166/166$ scored turns round-trip exactly through the
tokenizer. The median angle between correspondence-compatible group updates is $11.343^\circ$ and
the median relative displacement is $.204$; no witness reaches $45^\circ$ or has a negative inner
product. This audit is also \textsc{Weak}: correspondence measurably changes this parameter update,
but the evidence does not license a large-gradient, full-parameter-training, or policy-quality claim.

The interval law additionally passes $8{,}750$ randomized coordinate checks.

\paragraph{B.5\quad Independent compiler witness and randomization diagnostic.}
\emph{Randomization versus marginalization.}
We froze the two MatchTIR rollout samples, exact-optimum enumerator, uniform law over distinct optimal
reward vectors, materiality rule, and promotion thresholds before comparing
$\mathbb E[N(R_A)]$ with $N(\mathbb E[R_A])$. In the discovery sample, $7/20=35.0\%$ of material
groups contain an opposite-sign coordinate and the median maximum absolute gap is $.329$; in the
task-and-seed-disjoint sample the corresponding values are $8/22=36.4\%$ and $.241$. The pooled
cosines are nevertheless $.988$ and $.990$, and no held-out material group falls below cosine $.90$.
The pre-registered geometry gate therefore fails in both samples. The result establishes a replicated
local noncommutation diagnostic, not a large aggregate-gradient or final-policy effect.

\emph{LeTS source-level replication.}
LeTS \citep{zhang2025lets} constructs Jaccard process rewards for retrieval steps by one SciPy
Hungarian assignment, standardizes the selected rewards within a rollout, and uses them to rescale
group-relative learning credit. We froze the public source revision and exhaustively searched finite universes before
inspecting any witness. For three pairwise-distinct current and three pairwise-distinct gold
retrieval-result sets, the source-faithful Jaccard matrix
\[
 \begin{bmatrix}
 1 & 1/2 & 1/3\\
 0 & 0 & 1/3\\
 1/2 & 1/3 & 2/3
 \end{bmatrix}
\]
has two exact assignments of total score $5/3$. Their per-step rewards are
$(1,0,2/3)$ and $(1,1/3,1/3)$. On the released negative-outcome branch with effective
$\gamma=1$, the multiplier of the first fixed retrieval step is respectively $+.127$ and $-.155$.
The released SciPy call returns only the latter optimum. Independent recomputation matches the
extracted source function to floating-point precision.

This is an admitted source failure, not a prevalence estimate. The public training rollouts are not
released. Moreover, the paper describes coefficient $.1$, while the frozen executable path leaves
$\gamma=1$ at this operation and does not consume the two shell-level $.1$ fields inspected in the
audit. With coefficient $.1$, this witness changes multiplier magnitude but not sign. We therefore
do not infer which coefficient produced the paper's model, any effect on its reported results, or a
final-policy consequence.

\paragraph{B.6\quad Hidden exact-tie completion as a temporal credit policy.}
\label{app:positionaltie}
The all-optima audit above asks whether intended credit is identified. We next froze a separate
analysis asking which point the released implementation systematically selects from that fiber. This
analysis reuses the two already frozen rollout samples; it performs no new model generation and does
not alter their original preregistrations. Before computing an aggregate, we fixed the primary
population, intervention, task-cluster bootstrap, and four promotion criteria. Primary trajectories
have two to eight predicted calls, complete exact-vector enumeration, at least two distinct optimal
reward vectors, distinct predictions, and exact-optimum reward width at least $.10$.

\emph{Source mechanism and intervention.}
The released routine constructs score edges in row-major predicted-call order, applies Python's
stable descending score sort, and greedily accepts an edge whenever its prediction and target remain
free. The caller appends predicted calls chronologically. Hence score ties inherit temporal order as
an unstated secondary key. Let $\mathcal V^\star$ be the set of distinct exact-optimal per-call reward
vectors, let $\bar r=|\mathcal V^\star|^{-1}\sum_{v\in\mathcal V^\star}v$, and let
$r^{\rm src}$ be the released reward. The frozen intervention reverses only the row array given to
the selector and maps the resulting reward $r^{\rm rev}$ back to the original calls. It enters the
analysis only when both selectors attain the independently computed exact objective.

For normalized call position $z_i=i/(m-1)$, the trajectory bias is the slope of
$r_i^{\rm src}-\bar r_i$ on $z_i$; the reported contrast is the mean first-half minus last-half
value. The promotion rule required, independently in both samples: at least $95\%$ exact-objective
preservation; early-minus-late contrast at least $.05$ with task-clustered $95\%$ interval above
zero; at least $65\%$ negative within-trajectory slopes; and intended advantage-sign disagreement in
at least $20\%$ of affected four-rollout task groups.

\begin{center}
\small
\begingroup
\setlength{\tabcolsep}{2.5pt}
\begin{tabular}{@{}lrrrr@{}}
\toprule
Sample & $n$ & Exact & Bias [95\% CI] & Sign flips \\
\midrule
Discovery & 93 & $93/93$ & $.196\,[.149,.246]$ & $13/20$ \\
Held out & 90 & $90/90$ & $.174\,[.133,.219]$ & $14/22$ \\
\bottomrule
\end{tabular}
\endgroup
\end{center}

Negative slopes occur in $92/93$ discovery and $90/90$ held-out trajectories. Thus all frozen
criteria pass in both samples. This is not a solver-regret result: both the original and reversed
completions attain the exact matching score, and their total rewards differ by at most
$8.9\times10^{-16}$.

\emph{Independent hostile verification.}
An independent implementation and direct execution of four upstream functions agree on every
retained scorer, selector, and objective value; all interventions preserve call multisets, turns,
exact objectives, and totals. Re-aggregating over occupied turns gives contrasts
$.196\,[.148,.247]$ and $.173\,[.134,.217]$, and all three source strata retain the direction. A
post-hoc $128$-permutation placebo reduces mean bias to about $-.001$ in both samples. The same
$13/20$ and $14/22$ groups retain positive-to-negative reversals when both advantages must be at
least $.05$ from zero, ruling out a mere zero-crossing artifact.

\emph{Block characterization.}
If $p$ ordered rows tie at score $c$ for $q<p$ free columns, stable greedy completion credits the
first $q$ rows, reversed enumeration credits the last $q$, and the fiber mean assigns $qc/p$ to each;
all totals equal $qc$. The fact is elementary, but the frozen audits show that this hidden temporal
policy reverses strict normalized learning signs in a released compiler.

\emph{Interpretation and boundary.}
Chronology can be a legitimate secondary objective, but then it is part of the reward contract and
must be declared. The primary bipartite objective does not identify it. For a solver-invariant claim,
the all-optima sign mask avoids inheriting that hidden policy; for a point-valued operational
contract, a declared permutation-equivariant fractional or mixed completion is another option.
Neither result proves that an undocumented training run enabled the branch or that changing the
completion improves a final model.

\paragraph{B.7\quad Natural legal-update bodies and compiler interactions.}
\label{app:naturalupdatebody}
This post-hoc analysis does not revise the frozen consequence gate. It uses both rollout samples, the
exact reward enumerator, released compiler, and the complete $2{,}560$-coordinate final-RMSNorm scale
vector frozen in the original protocol. Other layers may have different geometry. For each task, we
compile every exact-optimal reward product and form its gradient convex hull.

The primary cache contains $15$ tasks and $571$ exact completion combinations. Median maximum
pairwise angle is $25.13^\circ$ (task-bootstrap $95\%$ CI $[6.87,51.87]$); $10/15$ tasks reach at
least $15^\circ$; median body diameter divided by median gradient norm is $.451$ (CI
$[.150,.802]$); and median affine rank is $7$. One task contains a negative-inner-product pair with
angle $161.90^\circ$. The canonical direction's scalar support crosses zero in one task. Applying the
fully corrective support-oracle procedure to the whole body gives $14$ strict common-progress
certificates and one abstention, uses at most nine oracle calls and six active vertices, and matches
the direct full-hull QP within $8.33\times10^{-12}$. The canonical direction is already strict common
progress on those same $14$ tasks. The robust direction changes no binary decision in this slice but
improves the worst-case unit margin by median $6.4\%$ and at most $28.3\%$. This comparison is
post hoc and establishes certificate quality, not a policy-performance gain.

A historically disjoint seven-task cache gives median maximum angle $17.20^\circ$, $4/7$ tasks above
$15^\circ$, median diameter-to-norm ratio $.316$, and common progress in all seven; this is a
temporal replication, not a preregistered test.

Exact functional ANOVA finds nonzero higher-order energy in all eight multi-fiber tasks under the
released standardizer (median $3.96\%$, range $[.54,10.33]\%$); freezing group statistics makes the
compiler affine and drives the maximum below $7.1\times10^{-30}$. Their interaction graphs have
median/max treewidth $2/3$. Factor reconstruction and min-sum elimination match full enumeration on
all $45/45$ frozen projected queries within $1.2\times10^{-19}$.

The frozen pre-registered consequence gate remains \textsc{Weak}: median angle
$7.551^\circ$ passed its directional threshold, normalized displacement $.1317$ missed $.15$,
$26.7\%$ exceeded $15^\circ$, and ambiguous mass was $11.6\%$. The post-hoc hull analysis changes
the estimand from one canonical contrast to all legal updates; it does not turn this weak result into
a policy-performance claim.

\paragraph{B.8\quad Source and complexity conformance audits.}
\label{app:compilerconformance}
The theorem in A.7 is analytic. Executable tests are included only to falsify algebra,
serialization, and implementation mistakes.
An independent implementation checks all $813$ no-isolate graphs through five vertices,
$25{,}264$ signs, and $160$ gadgets with zero failures. Direct source calls recover exactly two
vectors for $990$ binary gadgets and one for $360$ anchors. Literal call dictionaries preserve cut
ordering in all $744$ compiled assignments; the upstream PyTorch compiler matches the NumPy port on
$420$ random cases within $1.61\times10^{-6}$. Dyadic promise encodings through $m=2000$ use at most
$109$ denominator bits, and the full source/autograd chain preserves MAX-CUT minimizers through
$m=8$. These checks test conformance, not asymptotic hardness.

\begin{table*}[t]
\centering
{\small
\setlength{\tabcolsep}{4pt}
\begin{tabular}{@{}p{.31\textwidth}p{.25\textwidth}p{.37\textwidth}@{}}
\toprule
Conformance check & Audited population & Exact result \\
\midrule
Independent finite construction & 813 no-isolate graphs through five vertices; 25,264 signs; 160 gadgets & zero failures \\
Binary source fibers & 990 binary gadgets & exactly two returned vectors per fiber \\
Anchor source fibers & 360 anchors & exactly one returned vector per fiber \\
Literal call-dictionary compilation & 744 compiled assignments & cut ordering preserved in every assignment \\
Upstream compiler versus independent port & 420 random cases & maximum numerical discrepancy $1.61\times10^{-6}$ \\
Dyadic promise serialization & graph sizes through $m=2000$ & at most 109 denominator bits \\
Full source/autograd chain & graph sizes through $m=8$ & MAX-CUT minimizers preserved \\
\bottomrule
\end{tabular}
}
\caption{Complete conformance ledger for the compiler construction. These checks falsify algebraic,
serialization, and implementation errors in the finite construction; they are not evidence for
average-case hardness or for fixed-precision asymptotic hardness.}
\label{tab:compilerconformanceledger}
\end{table*}
The construction is intentionally synthetic and large: a loose serialized bound is
$O(m^9\log m)$ bits. The float checks do not establish fixed-precision asymptotic hardness.

\paragraph{B.9\quad Prospective transport audit.}
\label{app:pta}
This preregistered study is distinct from the banked audits. One \texttt{gpt-5.5}/\texttt{high}
configuration ran all $24$ paired trajectories and $72/72$ stages: two substrates, two views, six
fresh seeds, three horizons, $\tau=.05$, $\hat b=0$, and Bonferroni $\alpha=.05/2$. A failed partial
attempt contributes no data. The provider exposes neither immutable checkpoint digest nor
controllable sampling seed, so bitwise regeneration is not claimed.

\paragraph{B.10\quad Two-sided transport-admissibility gates.}
The evidence has three nonexchangeable levels: a failed original natural-response gate, a
post-diagnostic corrected-witness replication, and a later pre-outcome held-out validation. The
original $2$-repository $\times$ $4$-mutation gate included four nonexistent pylint test nodes, whose
``not found'' output was misparsed. We withdraw its cross-repository $8/8$ claim; only the four xarray
cells were executable. The invalid run is disclosed here and excluded from evidence.

The replacement call-site transport substitutes the local implementation body at its actual call
site, preserving wrapper-side preprocessing, postprocessing, exceptions, and return flow. In the
corrected V3 witness, benign bases pass before/after and map to zero; all eight harmful controls retain
identical failure signatures and positive distance. All four leave-one-family-out folds select this
transport with zero false invariance. Across the 36 natural states used for selection, it reduces mean
distance from $.1163$ to $.0244$ ($79.1\%$). Because the witness was chosen after diagnosis, V3 is
hypothesis-confirming, not independent discovery.

Before observing requests/pytest outcomes, we froze the implementation, candidate order, targets,
tests, four mutation families, and seven conditions. The external matrix passes all seven:
both benign views map to zero, all eight executable-harmful mutations preserve their failure
signatures and positive distance, while a destructive transport erases all eight. A separately
frozen family extension covers local alpha-renaming and branch inversion: all four benign pairs map
to zero, all $16/16$ harmful controls remain positive, and benign-perfect destructive maps erase
$10/16$. These results validate deterministic transfer across repositories and three transformation
families, not new natural trajectories or a second agent family.

The separately frozen natural-response gate applies the exact paper transport to $36$ evolved states.
All baselines pass; response deletion occurs in all $18$ xarray states and $0/18$ pylint states, so
the preregistered $75\%$ and both-substrate conditions fail. The paper claims natural deletion only
for xarray. A conservative exact-delegation guard accepts $2/2$ benign bases, rejects $8/8$ harmful
wrappers, and safely abstains on all $36$ evolved wrappers; the call-site transport covers all $36$
while preserving their witnesses.

\paragraph{B.11\quad Natural strictness failure and exact repair.}
\label{app:earthstrict}
\emph{Public contract and source defect.}
Earth-Agent places \emph{Tool-Exact-Match} in its step-by-step protocol and defines it as the longest
common prefix divided by the ground-truth length, while attributing lower Tool-Exact-Match values to
models that ``introduce irrelevant steps'' \citep{feng2026earthagent}. Under that definition, steps
appended after a complete correct prefix cost nothing. At the frozen source revision, the stepwise
scorer stops at the first positional mismatch and divides the matching prefix by the number of
expected calls. Its reported actual sequence is also truncated to the expected length. The
parameter-accuracy helper has the same one-sided denominator for full call equality.

\begin{restated}[Extension invariance and exact repair]
Let $r$ be a nonempty reference sequence, $p$ a predicted sequence, and $\ell(r,p)$ their exact
prefix length. The reference-normalized score $s_{\rm ref}(r,p)=\ell(r,p)/|r|$ assigns one to every
extension $r\mathbin{\|}z$. The two-sided score
$s_2(r,p)=\ell(r,p)/\max\{|r|,|p|\}$ assigns one if and only if $r=p$.
\end{restated}
\begin{proof}
For any suffix $z$, $\ell(r,r\mathbin{\|}z)=|r|$, proving the first statement. For the second,
$s_2=1$ requires $\ell=\max\{|r|,|p|\}$, while by definition
$\ell\leq\min\{|r|,|p|\}$; hence both lengths equal $\ell$ and every position agrees. Equality of
the complete sequences is sufficient by substitution. For an empty reference we define score one
only when the prediction is also empty.
\end{proof}
This repair is the minimum of reference coverage and prediction precision. We use it as a
conservative certificate for the source's advertised exactness contract, not as a new semantic
similarity metric. Replacing name equality with name-and-input equality gives the parameter version.

\emph{Frozen protocol.}
Before reading prediction contents, the audit fixed the repository version, the evaluator-defined RGB
slice (the final 59 of 247 ground-truth tasks), 26 named AP/IF configurations, at least $80\%$
coverage eligibility, the name and parameter repairs, task-clustered bootstrap, and six decision
gates. The preregistration and analysis were frozen before prediction contents were read.
All 26 configurations are eligible; two cells are missing, leaving $1{,}532$.

\begin{table*}[t]
\centering
{\footnotesize
\setlength{\tabcolsep}{2.5pt}
\begin{tabular}{@{}p{.14\textwidth}p{.14\textwidth}p{.09\textwidth}p{.25\textwidth}p{.12\textwidth}p{.20\textwidth}@{}}
\toprule
Audit & Status & Cells & False-perfect (clustered $95\%$ CI) & $\Delta\geq.05$ & Ranking effect \\
\midrule
Earth-Agent RGB & pre-registered GO & $1{,}532$ & $14.23\%$ $[7.58,21.80]\%$ & $26.17\%$ &
9 pair reversals; top changes \\
Earth-Agent common tasks & post-hoc robustness & $5{,}278$ & $5.67\%$ $[3.45,8.15]\%$ & $28.36\%$ &
18 pair reversals; top changes \\
MCPEval held out & pre-registered REDUCE & $248$ & $6.05\%$ $[2.82,9.68]\%$ & $16.53\%$ &
1 pair reversal \\
\bottomrule
\end{tabular}
}
\caption{Natural audits of one-sided strict trajectory scores. The Earth-Agent RGB row is
pre-registered and primary. The complete-case row was designed after that verdict and is labeled
post hoc. MCPEval is a weaker independent replication whose binary strict-success field already
checks extra calls.}
\label{tab:naturalstrictness}
\end{table*}

\emph{Primary result and parity.}
The pre-registered audit finds 218 name-level and 212 parameter-level false-perfect cells. Name
defects occur in $22/26$ configurations and $12/13$ model families. The repaired name score changes
nine pairwise comparisons and the top configuration from DeepSeek-V3.1 IF to GPT-4o AP
(Kendall $\tau_b=.9446$). All exactness checks pass with zero repair errors. An independent
source-parity check reproduces every saved official task score with maximum absolute difference zero. In
post-hoc stress tests, the top change survives all $59$ leave-one-task-out omissions; leave-one-family
false-perfect prevalence stays in $[13.7,15.4]\%$.

\emph{Full-corpus and independent checks.}
The repository also contains predictions beyond the source evaluator's RGB slice. An explicitly
post-hoc available-case analysis covers 6,374 cells. To remove differential missingness, an
all-configuration complete-case analysis uses the 203 tasks observed for every configuration and
finds the second row of \Cref{tab:naturalstrictness}: defects remain in $23/26$ configurations and
all 13 families, while the top configuration changes from Kimi-K2 IF to GPT-4o AP.

MCPEval \citep{liu2025mcpeval} states that strict matching requires a generated trajectory exactly
equal to ground truth and mainly reports a weighted overall score. At the frozen source revision,
its binary strict-success field correctly
checks extra calls, but name, parameter, and LCS order components use one-sided normalization inside
the reported overall score. A discovery model was excluded before inspecting four held-out models.
The held-out result in \Cref{tab:naturalstrictness} misses its frozen prevalence and materiality
gates, so it is supporting evidence only.

Independent reruns reproduce the frozen results exactly. These audits use no model generation, LLM
judge, or human label. They correct the declared
strict trajectory metric; they do not establish that extra calls are semantically harmful, alter
end-to-end answer accuracy, or invalidate either source framework as a whole.

\section{Extended Limitations and Scope}

\paragraph{Contract-relative validity.}
Two-sided certification is relative to the declared null controls, positive controls, behavioral
witnesses, and transport family. Passing those checks does not prove global semantic equivalence.
The linear subspace theorem gives a sharp feasibility boundary for linear transports; it motivates,
but does not by itself validate, the nonlinear source transformations used in the executable audits.
Those transformations are validated by frozen tests and witnesses instead.

\paragraph{Natural-audit population.}
The temporal-alignment result concerns the released code and SQL trajectory pairs after excluding
zero-edit pairs. The MatchTIR fractions condition on task groups with material exact-optimum reward
width and are not estimates of ambiguity prevalence in an unrestricted deployment distribution.
The discovery and task-and-seed-disjoint samples are reported separately, and the natural update-body
analysis is post hoc rather than a pre-registered training comparison.

\paragraph{Learning-consequence boundary.}
The MatchTIR audit identifies turn rewards and intended advantage signs compatible with all
exact-optimal correspondences on frozen post-training rollouts. The public quick-start does not
visibly enable the intended multi-turn mask, and no published training trace exposes the assignment
chosen at every update. Even with the branch enabled, a local sign reversal need not reduce final
performance because later updates, trajectory-level credit, and optimization may compensate. A
policy-level claim requires a frozen comparison of declared tie policies under one active training
configuration. We report no such effect, and the pre-registered downstream consequence remains
\textsc{Weak}.

\paragraph{Positive-control evidence levels.}
The prospective natural-response gate, the corrected-witness replication, and the later pre-outcome
held-out validation are not exchangeable samples. The first failed, the second followed diagnosis,
and only the third tests the frozen transport-selection rule on held-out repositories. They are never
pooled into one success rate. The held-out result validates deterministic code transformations across
four repositories, not new natural trajectories or a second agent family.

\paragraph{Complexity boundary.}
The shared-standardization reduction is a worst-case, finite-bit construction. Executable conformance
checks falsify algebra and source-realizability errors but do not establish natural-case hardness or
fixed-precision asymptotic hardness. The natural interaction graphs audited here have low treewidth
and are exactly solvable. The theorem therefore separates compiler regimes; it does not claim that
ordinary MatchTIR groups instantiate hard MAX-CUT cases.

\paragraph{Generality.}
Most reanalyses are deterministic from released or frozen samples. Provider-side agent outputs cannot
be reproduced bitwise because immutable checkpoint digests and controllable sampling seeds were not
available. The prospective transport audit uses one agent family, while the MatchTIR analysis uses one
released checkpoint and one declared compiler family. Cross-model and cross-training-algorithm
generality remain untested.
\section{Computational Cost of Solver-Independent Auditing}

The body audits report an envelope over an optimizer's optimal set rather than a solver-returned point.
Theorem~2 proves one compiler-dependent hardness boundary. This section introduces no additional
complexity classification; it locates the diagnostics actually audited within the known landscape and
gives a reporting rule when a backend leaves the tractable region.

\paragraph{D.1\quad Detection cost.} For an additive readout $h$, deciding whether a point diagnosis
is solver-determined costs no more than its sharp envelope: the same two-pass dynamic program that
returns $\min$ and $\max$ over the optimal-path DAG (Raj) decides constancy by comparing them, and an
additive readout on an assignment face requires two min/max-cost matchings on the equality graph.
For a nonlinear readout such as inversion-based Order Consistency, assignment multiplicity remains
easy to detect, but multiplicity alone does not prove that $h$ varies and no general two-matching
claim is made. B.2 exactly enumerates the small public tie faces used in this paper.

\paragraph{D.2\quad Polynomial envelopes for the audited backends.} For an additive diagnostic on a
shortest-path DAG the endpoints are a two-pass DP; for an additive diagnostic on the optimal-assignment
face the endpoints are two min/max-cost matchings, and that face is an integral polytope (Birkhoff--von
Neumann \citep{birkhoff1946tres}), so the extrema are attained at genuine matchings. Both audited pipelines
fall in this region.

\paragraph{D.3\quad Complexity outside the audited backends.} Two standard results bound what a
general solver-independent summary can cost; we cite them as the surrounding landscape, not as our
contribution. First, once ``optimal'' is relaxed to \emph{near}-optimal ($f\le f^\star+\varepsilon$),
optimizing an additive diagnostic over the enlarged set---the two additive edge weights being cost and
diagnostic---is in general as hard as the resource-constrained shortest-path problem, which is NP-hard
\citep{handler1980dual}: exact-tie auditing can be easy while approximate-tie auditing is not. Second,
\emph{exactly} averaging even a linear diagnostic uniformly over an optimal matching set---a uniform
edge-indicator average is the edge marginal $\mathrm{per}(\cdot)/\mathrm{per}$---is in general as hard
as evaluating a permanent, hence \#P-hard \citep{valiant1979permanent}, though it admits a fully
polynomial randomized approximation scheme \citep{jerrum2004permanent}. Guarding (the sharp $\min/\max$) can thus be polynomial on the same fiber
where exact fair averaging is intractable.

\paragraph{D.4\quad Reporting rule.} Report the exact envelope where a polynomial backend applies, as both
body audits do. Otherwise report a declared randomized estimate with its sampling error, or abstain,
rather than reporting a single hidden tie-break as the diagnosis. This is guidance for deploying the
auditor, not a claim that the paper's diagnostics are hard: the cases we audit are exactly the tractable
ones.

\paragraph{AI assistance.} Generative AI supported language editing, exploratory code drafting, and
literature search; the authors verified and retain responsibility for all content.

\bibliographystyle{plainnat}
\bibliography{references}

\end{document}